%% file: main.tex
\documentclass{article}
\usepackage{ijcai24}

\usepackage{times}
\usepackage{soul}
\usepackage{url}
\usepackage[hidelinks]{hyperref}
\usepackage[utf8]{inputenc}
\usepackage[small]{caption}
\usepackage{graphicx}
\usepackage{amsmath}
\usepackage{amsthm}
\usepackage{booktabs}
\usepackage{algorithm}
\usepackage{algorithmic}
\usepackage[switch]{lineno}
\usepackage{tikz}
\usepackage{xcolor}

\usepackage{listings}

\input{macros}

\newcommand{\prog}{\PG}

\title{Finite Groundings for ASP with Functions: A Journey through Consistency\\(Technical Report)}

\author{
Lukas Gerlach$^1$
\and
David Carral$^2$
\and
Markus Hecher$^3$
\affiliations
$^1$Knowledge-Based Systems Group, TU Dresden, Dresden, Germany\\
$^2$LIRMM, Inria, University of Montpellier, CNRS, Montpellier, France\\
$^3$Massachusetts Institute of Technology, United States
\emails
lukas.gerlach@tu-dresden.de,
david.carral@inria.fr,
hecher@mit.edu
}
\newcommand{\la}{\ensuremath{\leftarrow}}

\newcommand{\inputPredColor}{purple!70!black}
\newcommand{\outputPredColor}{orange!70!black}
\newcommand{\specialTermColor}{blue!70!black}

\lstdefinelanguage{dflat}{
    firstnumber=1, % XXX Workaround because otherwise warnings occur with hyperref
    otherkeywords={:-},
    morekeywords={not},
    keywordstyle=\sffamily\bfseries,
    emph={win,position,passenger,transport,goAlone,takeSome,eats,opposite,othertransport,change,steps,bank,winEnd,lose,redundant},
    moreemph=[2]{},
    moreemph=[3]{farmer,east,west,wolf,goat,cabbage},
    moreemph=[4]{},%out, in, outc, def, defc},
    alsoletter={\#}, % This is required to avoid coloring, e.g., the #count aggregate because it is mistakenly interpreted as the count/1 predicate
    emphstyle=\color{\inputPredColor},
    emphstyle=[2]\color{\outputPredColor},
    emphstyle=[3]\color{\specialTermColor},
    emphstyle=[4]\color{\decompositionColor},
    morecomment=[l]{\%},
    commentstyle=\rmfamily\small\itshape\color{dkgreen},
    numbers=left,
    numbersep=2pt,                  % how far the line-numbers are from the code
    numberstyle=\tiny\color{gray},
    numberblanklines=false,
    literate={:-}{{$\la$}}2 {!=}{{$\neq$}}1,
    breakindent=20pt,
    escapechar=@, % Useful, e.g., for manually breaking lines via @\\@
}

\begin{document}

\maketitle

\begin{abstract}
  Answer set programming (ASP) is a
  logic programming formalism 
  used in various areas of artificial intelligence 
  like combinatorial problem solving and 
  knowledge representation and reasoning.
  It is known that enhancing ASP with function symbols makes basic reasoning problems 
  highly undecidable.
  However, even in simple cases, state of the art reasoners,
  specifically those relying on a ground-and-solve approach,
  fail to produce a result.
  Therefore, we reconsider consistency as a basic reasoning problem for ASP. 
  We show reductions that give an intuition for the high level of undecidability. 
  These insights allow for a more fine-grained analysis 
  where we characterize ASP programs as ``frugal'' and ``non-proliferous''.
  For such programs, we are not only able to semi-decide consistency 
  but we also propose a grounding procedure that
  yields finite groundings on more ASP programs
  with the concept of ``forbidden'' facts.
\end{abstract}

\input{sections/1-introduction}

\input{sections/2-preliminaries}

\input{sections/3-consistency}

\input{sections/4-termination}

\input{sections/5-reasoning-procedures}

\input{sections/6-conclusion}

\appendix

%\section*{Ethical Statement}

%There are no ethical issues.

\section*{Acknowledgments}

We want to acknowledge that the full modelling of the wolf, goat cabbage puzzle from the introduction 
is inspired by lecture slides created by Jean-François Baget.

% Funding that Lukas should put for KBS group
On TU Dresden side, this work was partly supported
by Deutsche Forschungsgemeinschaft (DFG, German Research Foundation) in project 389792660 (TRR 248, \href{https://www.perspicuous-computing.science/}{Center for Perspicuous Systems});
by the Bundesministerium für Bildung und Forschung (BMBF) in the \href{https://www.scads.de}{Center for Scalable Data Analytics and Artificial Intelligence} (ScaDS.AI);
by BMBF and DAAD (German Academic Exchange Service) in project 57616814 (\href{https://secai.org/}{SECAI}, \href{https://secai.org/}{School of Embedded and Composite AI});
and by the \href{https://cfaed.tu-dresden.de}{Center for Advancing Electronics Dresden} (cfaed).

Carral was financially supported by the ANR project CQFD (ANR-18-CE23-0003).

Hecher is funded by the Austrian Science Fund (FWF), grants J 4656 and P 32830, the Society for Research Funding in Lower Austria (GFF, Gesellschaft für Forschungsf\"orderung N\"O) grant ExzF-0004, as well as the Vienna Science and Technology Fund (WWTF) grant ICT19-065.
Parts of the research were carried out while Hecher was visiting the Simons institute for the theory of computing at UC Berkeley.

\bibliographystyle{named}
\bibliography{main}

\clearpage

\input{sections/a-examples}

\input{sections/b-proofs-consistency}
\input{sections/c-proofs-termination}

\input{sections/d-proofs-reasoning-procedures}

\end{document}

%% file: macros.tex
\usepackage{amssymb}
\usepackage{amsthm}
\usepackage{enumerate}
\usepackage{enumitem}
\usepackage{mathtools}
\usepackage{todonotes}
\usepackage{xspace}

\makeatletter
\newtheorem*{rep@theorem}{\rep@title}
\newcommand{\newreptheorem}[2]{%
\newenvironment{rep#1}[1]{%
\def\rep@title{#2 \ref{##1}}%
\begin{rep@theorem}}%
{\end{rep@theorem}}}
\makeatother

\newtheorem{definition}{Definition}
\newtheorem{example}{Example}
\newtheorem{proposition}{Proposition}
\newtheorem{theorem}{Theorem}
\newtheorem{lemma}{Lemma}
\newreptheorem{theorem}{Theorem}
\newreptheorem{lemma}{Lemma}
\newreptheorem{corollary}{Corollary}

\newcommand{\FormatEntitySet}[1]{\ensuremath{\mathtt{#1}}}
\newcommand{\FormatFormulaSet}[1]{\ensuremath{\mathcal{#1}}}
\newcommand{\FormatFunction}[1]{\ensuremath{\mathsf{#1}}}

\renewcommand{\vec}[1]{\boldsymbol{#1}}

\newcommand{\FunctionSymbols}{\FormatEntitySet{Funs}\xspace}
\newcommand{\Funs}{\FunctionSymbols}
\newcommand{\FunsOfArity}[1]{\ensuremath{\Funs_{#1}}\xspace}
\newcommand{\Variables}{\FormatEntitySet{Vars}\xspace}
\newcommand{\Vars}{\Variables}
\newcommand{\Predicates}{\FormatEntitySet{Preds}\xspace}
\newcommand{\Preds}{\Predicates}
\newcommand{\PredsOfArity}[1]{\ensuremath{\Predicates_{#1}}\xspace}
\newcommand{\Constants}{\FormatEntitySet{Cons}\xspace}
\newcommand{\C}{\Constants}
\newcommand{\Terms}{\FormatEntitySet{Terms}\xspace}

\newcommand{\Arity}{\ensuremath{\FormatFunction{ar}}\xspace}

\newcommand{\Expression}{\ensuremath{\phi}\xspace}

\newcommand{\Rule}{\ensuremath{\rho}\xspace}

\newcommand{\HerbrandUniverse}{\FormatFunction{HUniv}\xspace}
\newcommand{\HU}{\HerbrandUniverse}
\newcommand{\Ground}{\FormatFunction{Ground}\xspace}

\newcommand{\PG}{\ensuremath{\FormatFormulaSet{P}}\xspace}

\newcommand{\Substitution}{\ensuremath{\sigma}\xspace}
\newcommand{\Subs}{\Substitution}

 \newcommand{\NPTimeC}{\textsc{NP-}complete\xspace}

 \newcommand{\NExpTimeC}{\textsc{NExpTime-}complete\xspace}

 \newcommand{\Naturals}{\ensuremath{\mathbb{N}}\xspace}

 \allowdisplaybreaks

%% file: sections/1-introduction.tex
\section{Introduction}
\label{section:introduction}

Answer set programming~\cite{BrewkaEiterTruszczynski11,GebserKaminskiKaufmannSchaub12} is a logic-based formalism 
used in multiple fields of artificial intelligence research such as
knowledge representation and reasoning
but also combinatorial problem solving. 
State-of-the-art ASP solvers like clasp~\cite{GebserKaufmannSchaub09a} or wasp~\cite{AlvianoEtAl22} rely on a ground-and-solve approach.
During (i) \emph{grounding}\footnote{In the introduction, \emph{grounding} (informally) refers to the result of a (naive) grounding procedure. We formalize this later on.} 
a given 
%(non-ground) 
ASP program is instantiated with all relevant terms. 
% thereby generating every relevant (ground) ASP rule. 
Then, when (ii) \emph{solving} the ground program, the ASP solver, 
which is a SAT solver extended by unfounded set propagation, 
excludes sets of atoms that lack foundation (i.e. unfounded sets),
thereby efficiently computing answer sets.
Unfortunately, with function symbols involved, 
already the grounding step may not~terminate.
% even if a finite answer set exists.

\begin{example}\label{exp:artificial}
  The program $\{\eqref{rule1}, \eqref{rule2}, r(a, b)\}$ admits exactly one answer set: $\{ \mathit{r}(a, b), \allowbreak \mathit{stop}(b), \allowbreak r(b, f(b)), \allowbreak \mathit{stop}(f(b))\}$.
  Still, the grounding is infinite with terms $b, f(b), f(f(b)), \ldots$
\begin{align}
    \mathit{r}(Y, f(Y)) &\leftarrow \mathit{r}(X, Y), \neg \mathit{stop}(X).  \label{rule1}\\
    \mathit{stop}(Y) &\leftarrow \mathit{r}(X, Y). \label{rule2}
\end{align}
\end{example}

Such problems indeed manifest in real world applications e.g. in a knowledge respresentation contexts. 
One prominent example is an approach for simulating sets 
in ASP (using function symbols) \cite{ASPWithSets}.
%
% While Example~\ref{exp:artificial} seems artificial,
% there are different areas of interest where this can 
% become a problem in practice.
% %
% First, in an ontological reasoning context 
% one can use ASP to support negation as failure. 
% Necessary value inventions in this context require function symbols. 
%
% \begin{example}\label{exp:partonomy}
%   We can model a simple partonomy, e.g. every bicycle (\textit{B}) has a part (\textit{HP}) wheel (\textit{W}) and every wheel is part of (\textit{IPO}) a bicycle 
%   with the following program $P$ in ASP.\footnote{We model the behavior of the restricted chase with ASP using the auxilliary predicates \textit{HPW} and \textit{IPOB}.}
%   \begin{align*}
%     \mathit{HP}(x, \mathit{wh}(x)), \mathit{W}(\mathit{wh}(x)) &\leftarrow \mathit{B}(x), \lnot \mathit{HPW}(x). \\
%     \mathit{HPW}(x) &\leftarrow \mathit{HP}(x, y), \mathit{W}(y), y \neq \mathit{wh}(x). \\
%     \mathit{IPO}(x, \mathit{bc}(x)), \mathit{B}(\mathit{bc}(x)) &\leftarrow \mathit{W}(x), \lnot \mathit{IPOB}(x). \\
%     \mathit{IPOB}(x) &\leftarrow \mathit{IPO}(x, y), \mathit{B}(y), y \neq \mathit{bc}(x). \\
%     \mathit{HP}(y, x) \leftarrow \mathit{IPO}(x, y). \quad
%     &\mathit{IPO}(y, x) \leftarrow \mathit{HP}(x, y). \quad
%     \mathit{B}(c).
%   \end{align*}
%   $P$ has the following answer set but the ground-and-solve approach fails to terminate.
%   $\{ \mathit{B}(c),\allowbreak \mathit{HP}(c, wh(c)),\allowbreak \mathit{IPO}(wh(c), c),\allowbreak \mathit{IPOB}(wh(c)),\allowbreak \mathit{W}(wh(c)) \}$
% \end{example}
%
%Second, 
For combinatorial problems, bounds on the size of natural numbers (which could be modelled using function symbols)
are often introduced to ensure termination of the ground-and-solve approach. 
This is observable in many typical ASP examples. 

\begin{figure}[t]
  \centering
  \includegraphics[width=7cm,height=1.5cm]{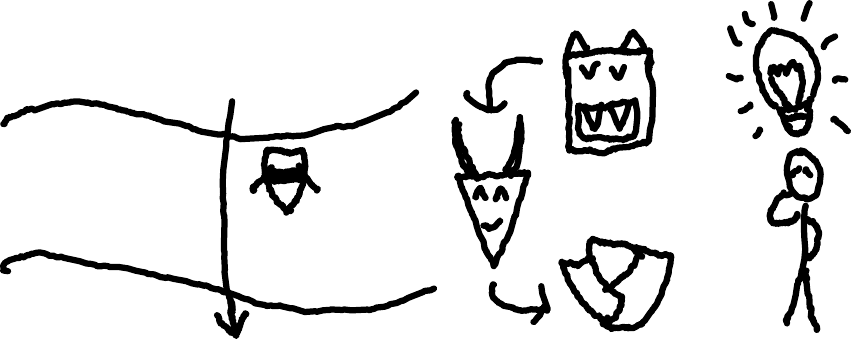}
  \caption{Wolf, Goat, Cabbage Puzzle}\label{fig:wolf-goat-cabbage}
\end{figure}

\begin{example}\label{exp:wolf-goat-cabbage-intro}
  We consider the famous puzzle of a farmer who needs to cross a river with a wolf, a goat, and a cabbage.
  They may only take one item at a time and must not leave the wolf and the goat or the goat and the cabbage alone 
  since then the former will eat the latter (see Figure~\ref{fig:wolf-goat-cabbage}).
  One essential part of the considered modelling is a rule as the following together with enough generated atoms, e.g. 
  $\mathit{steps}(0...100)$.\footnote{Full example in the technical appendix.} % Listing~\ref{lst:wolf-goat-cabbage}.
  \begin{align*}
    &\mathit{position}(X, C, N+1) \leftarrow \mathit{transport}(X, N), \\ 
    & \qquad \mathit{position}(X, B, N), \mathit{opposite}(B, C), \mathit{steps}(N+1) .
  \end{align*}

  That is, if we guess that item $X$ is transported in step $N$, then its position is updated to the opposite river bank if we are not out of steps yet. 
  Additional rules are introduced to detect and avoid redundant positions; we ellaborate on this idea later
  in Example~\ref{exp:essential-wgc-different}.
  % \begin{align*}
  %   &\mathit{change}(N, M) \leftarrow \mathit{position}(X, B, N), \\ 
  %   & \qquad \mathit{position}(X, C, M), \mathit{opposite}(B, C), N < M. \\ 
  %   &\mathit{redundant} \leftarrow \mathit{position}(X, B, N), \\ 
  %   & \qquad \mathit{position}(X, B, M), N < M, \lnot \mathit{change}(N, M). \\
  %   &\leftarrow \mathit{redundant}.
  % \end{align*}
  However, despite the redundancy check, we need to bound (guard) the term $N{+}1$, %of $\mathit{steps}(N{+}s1)$, % rule body 
  as otherwise the grounding is infinite.
  Such guards are common in~ASP. 
\end{example}

\paragraph{Related Work.}
Due to the complications with function symbols,
some works avoid them altogether \cite{marek_effectively_2011}.
However, some existing reasoning approaches seem promising.
We observe that lazy-grounding, as used by Alpha~\cite{WeinzierlEtAl20}, 
achieves termination on more programs than ground-and-solve approaches.
For example, Alpha yields the expected finite answer set in Example~\ref{exp:artificial} 
%and \ref{exp:partonomy}
but still fails in Example~\ref{exp:wolf-goat-cabbage-intro} (without the $N{+}1$ guard).
As an extension of ground-and-solve approaches, incremental solving have been proposed~\cite{GebserEtAl19}, 
where one can increment the maximal number of steps used for grounding and interatively reground. 
Thereby, one could prevent the issue of infinite groundings when only finite answer sets exist.
Various efforts have also gone into characterizing ASP programs into classes 
that e.g. yield finite groundings \cite{alviano_function_2012}.
One particular idea defines the semi-decidable class of \emph{finitely ground} programs 
including the decidable restriction of \emph{finite domain} programs \cite{ComputableFuncsInASP}. 
This approach has been implemented in the DLV solver~\cite{AlvianoEtAl10}.
Still, we observe that (i) DLV~\cite{CalimeriEtAl17} does not (seem to) terminate on 
Examples~\ref{exp:artificial} 
%\ref{exp:partonomy}, 
and \ref{exp:wolf-goat-cabbage-intro},
(without the $N{+}1$ guard). 
Grounding is active research, %. % in general. 
%The literature distinguishes many approaches, %, which is an active field of research,
ranging from traditional instantiation \cite{KaminskiSchaub21}, %over
%structural techniques%new techniques using tree decompositions
%~\cite{BichlerMorakWoltran20,CalimeriPerriZangari19}, 
over size estimations \cite{HippenLierler21}, lazy grounding \cite{WeinzierlEtAl20}, %\cite{%BogaertsWeinzierl18,
%BomansonJanhunenWeinzierl19,WeinzierlTaupeFriedrich20}. %care that modeling
%
%Both aspects (i) and (ii) combined finally lead to ASP being successfully applied in academia as well as industry. %~\cite{GebserEtAl11}. 
%
%constraint-programming extensions~\cite{} or ASP modulo theory~\cite{}.
%
%There are further attempts to avoid the grounding bottleneck, %The literature distinguishes techniques to circumvent this bottleneck~\cite{CuteriEtAl20,TsamouraEtAl20}, like the aforementioned lazy grounding and 
%like %the usage of 
%constraint-programming or 
ASP modulo theory~\cite{BanbaraEtAl17,JanhunenEtAl17,CabalarEtAl20}, 
%as well as 
and treewidth-based methods~\cite{BichlerMorakWoltran20}.

\paragraph{Contributions.} 
We aim to improve existing reasoning techniques further in terms of termination.
As a prerequisite, a better understanding of the hardness of reasoning is required. 
% Our approach is as follows.
\begin{itemize}
  \item In Section~\ref{sec:consistency}, we consider consistency as our exemplary highly undecidable ($\Sigma_1^1$-complete) reasoning problem. 
    We give easy to follow reductions that give an intuition into the cause of the high level of undecidability.
    % We provide simpler and fully formal proofs for the undecidability of non-ground ASP in case of function symbols. Existing approaches heavily relied on Goedel numbers, whereas we show hardness by reducing from a Tiling problem, thereby providing a better intuition on problems in the undecidability class $\Sigma_1^1$. For membership in this class, we design a non-deterministic Turing machine that will be the basis for contributions below.
	%
	%After infinitely many steps, but arbitrary number of steps
	%
  \item Based on these studies, in Section~\ref{sec:termination}, 
    we characterize ASP programs by two essential causes for undecidability of reasoning and infinite groundings. 
    %we provide a deeper study on the undecidability of non-ground ASP with a more fine-grained distinguishment between infintely many and infinitely large answer sets. %
	  Surprisingly, even if a program is 
    \emph{frugal} (only finite answer sets) and 
    \emph{non-proliferous} (only finitely many finite answer sets), 
    we still obtain undecidability for program consistency.
  \item To still tackle consistency, in Section~\ref{sec:reasoning-procedures}, 
    we propose a semi-decision algorithm for 
    frugal and non-proliferous programs that also terminates on many inconsistent programs.
    Based on the underlying idea of \emph{forbidden} atoms, we moreover define a grounding procedure 
    that terminates in more cases; like the above examples.
    % of the question whether a program admits a finite answer sets. It turns out that one can adapt our non-deterministic Turing machine of the first contribution to obtain a semi-decision procedure for frugal programs. 
	%
	% 
	%We improve this approach by considering forbidden facts that never appear in any answer set. 
	%
	% Our semi-decision procedure immediately yields a finite grounding and can therefore be embedded in a ground-and-solve framework.
	%
	%
	%\item %
\end{itemize}

%To this aim, we structure the paper as follows. 
%
%\begin{itemize}
%  \item Section~X: We formalize a notion of termination for ASP programs. 
%  \item Section~X: We show a sufficient condition for termination that yields a bound on the depth of functional terms.
%  \item Section~X: We demonstrate our approach on the forementioned examples.
%\end{itemize}

%% file: sections/2-preliminaries.tex
\section{Preliminaries}
\label{section:preliminaries}

We assume familiarity with propositional satisfiability (SAT)~\cite{KleineBuningLettman99,BiereHeuleMaarenWalsh09},
where we use clauses, formulas, and assignments as usual.

\paragraph{Ground Answer Set Programming (ASP).}
We follow standard definitions of propositional ASP~\cite{BrewkaEiterTruszczynski11,JanhunenNiemela16a}.
%For comprehensive foundations, we refer to introductory
%literature~\cite{BrewkaEiterTruszczynski11,JanhunenNiemela16a}.
%
Let $\ell$, $m$, and $n$ be non-negative integers with $0 \leq \ell$, $0 \leq m \leq n$, and let $b_1, \ldots, b_\ell, a_1, \ldots a_n$ be propositional atoms.
Moreover, a \emph{literal} is an atom or its negation.
%, and
%$l \in \{a_1, \neg a_1\}$ being a \emph{literal},~i.e., an atom or the negation thereof.
%
A \emph{ground rule} $r$ is an implication of the form $b_1, \ldots, b_\ell \leftarrow a_{1}, \ldots, a_{m}, \neg
a_{m+1}, \ldots, \neg a_n$ where $0 \leq \ell \leq 1$; that is, a formula with at most one atom before $\leftarrow$. % with~$\ell\leq 1$.
For such a rule, we define $H_r = \{b_1,\ldots, b_\ell\}$ and $B_r = \{a_{1}, \ldots, a_{m}, \lnot a_{m+1}, \ldots, \lnot a_n\}$.
Note that ground rules are non-disjunctive since $|H_r|\leq 1$; for \emph{constraints} we have $|H_r|=0$.
Moreover, for a set of literals $L$ (such as $B_r$), let $L^+$ be the set of all positive literals in $L$ and let $L^- = \{a \mid \lnot a \in L\}$.
% We denote the sets of \emph{atoms} occurring in a rule~$r$ or in a
% program~$\prog$ by $\at(r) = H_r \cup B^+_r \cup B^-_r$ and
% $\at(\prog)= \bigcup_{r\in\prog} \at(r)$, respectively. 
%
%\begin{align*}
%\smallskip
%\hill
%\vspace{-0.75em}
%\(
%a_1\vee \cdots \vee a_\ell 
%a \leftarrow a_{1}, \ldots, a_{m}, \neg
%a_{m+1}, \ldots, \neg a_n.
%
%\)
%
% \emph{Constraints} are of the form $\bot \leftarrow a_{1}, \ldots, a_{m}, \neg
% a_{m+1}, \ldots, \neg a_n$, which is a shortcut for the ground rule $c \leftarrow a_{1}, \ldots, a_{m}, \neg
% a_{m+1}, \ldots, \neg a_n, \neg c$ where $c$ is a fresh atom.
%\hfill
%
%\noindent 
%For a ground rule~$r$, we let $H_r = \{a\}$,
%$B_r = \{a_{1}, \ldots, a_{m}, \lnot a_{m+1}, \ldots, \lnot a_n\}$, and
%$B^-_r = \{a_{m+1}, \ldots, a_n\}$.
%
%
%
A \emph{(normal) ground program} is a set of ground rules (and constraints).

An \emph{interpretation} $I$ is a set of atoms. 
An interpretation $I$ \emph{satisfies} a
ground rule~$r$ if $(H_r\,\cup\, B^-_r) \,\cap\, I \neq \emptyset$ or
$B^+_r \setminus I \neq \emptyset$; it is a \emph{model} of a ground program $\prog$
if it satisfies all rules is~$\prog$. %, in symbols $I \models \prog$. %, if $I$ satisfies every rule~$r \in \prog$.
For a set~$A$ of atoms, a function~$\varphi: A \rightarrow \Naturals$ 
is an \emph{ordering} over~$A$.
Consider a model $I$ of a ground program~$\prog$, and an ordering $\varphi$ over~$I$.
An atom~$a \in I$ is \emph{proven (justified)} 
if there is a ground rule~$r\in\prog$ with $a\in H_r$ such that (i)~$B^+_r\subseteq I$,
(ii)~$I \cap B^-_r = \emptyset$ and $I \cap (H_r \setminus \{a\}) = \emptyset$,
as well as (iii)~$\varphi(b) < \varphi(a)$ for every~$b\in B_r^+$.
Then, $I$ is an
\emph{answer set} of~$\prog$ if (I)~$I$ is a model of~$\prog$, and
(II) \emph{$I$ is proven}, i.e., every~$a \in I$ is proven. 
Deciding whether a ground program has an answer set is called
\emph{consistency}, which is 
\NPTimeC~%~\cite{BidoitFroidevaux91,
\cite{MarekTruszczynski91} for normal ground programs.
For non-ground normal programs without functions, this is \NExpTimeC~\cite{EiterGottlobMannila94}; 
$\Sigma_2^\textsc{P}$-complete with bounded arities~\cite{EiterFaberFink07}.
%
%
% \begin{example}
% Consider the program
% \end{example}

%~\\[-2em]
\paragraph{Non-Ground ASP.}
We define \Preds, \Funs, \Constants, and \Vars to be mutually disjoint and countably infinite sets of predicates, function symbols, constants, and variables, respectively.
Every $s \in \Preds \cup \Funs$ is associated with some \emph{arity} $\Arity(s) \geq 0$.
For every $i \geq 0$, both $\PredsOfArity{i} = \{P \in \Preds \mid \Arity(P) = i\}$ and $\FunsOfArity{i} = \{f \in \Funs \mid \Arity(f) = i\}$ are countably infinite.
The set $\Terms$ of terms includes $\C$ and $\Vars$; and contains $f(t_1, \dots, t_i)$ for every $i \geq 1$, every $f \in \FunsOfArity{i}$, and every $t_1, \ldots, t_i \in \Terms$.
A term $t \notin \Variables \cup \Constants$ is \emph{functional}.
%For a term $t$; let $\Depth(t) = 1$ if $t$ is not functional, and $\Depth(t) = 1 + \Max(\Depth(s_1), \ldots, \Depth(s_n))$ if $t$ is of the form $f(s_1, \ldots, s_n)$.
%We write lists $t_1, \ldots, t_n$ of terms as $\Vt$, oftentimes treated as sets.
A \emph{(ground) substitution} is partial function from variables to \emph{ground terms}; that is, to variable-free terms.
We write $[x_1 / t_1, \dots, x_n / t_n]$ to denote the substitution mapping $x_{1}, \dots, x_{n}$ to $t_{1}, \dots, t_{n}$, respectively.
For an expression $\Expression$ and a substitution \Subs, let $\Expression\Subs$ be the expression resulting from $\Expression$ by uniformly replacing every syntactic occurrence of every variable $x$ by $\Subs(x)$ if defined.

Let $q_1,\dots,q_\ell,p_1,\ldots,p_n \in \Preds$ be predicates and $S_1,\ldots,S_\ell,T_1, \ldots, T_n$ be vectors over $\Terms$. 
%, where each takes \emph{arity}~$|p_i|$ 
%many variables for~$1\leq i \leq n$.
%
A \emph{(non-ground) program} $\prog$ is a set of \emph{(non-ground) rules} of the form
$H \leftarrow p_{1}({T_{1}}), \ldots, p_{m}({T_{m}}), \neg p_{m{+}1}({T_{m{+}1}}), \ldots, \neg p_n({T_n})$ 
where 
$H = q_1({S_1}), \dots, q_\ell({S_\ell})$ with $0 \leq \ell \leq 1$, 
$|{{S_i}}|{=}\Arity({q_i})$ for every $1 \leq i \leq \ell$, 
and $|{{T_i}}|{=}\Arity({p_i})$ for every $1 \leq i \leq n$.
Note again that we do not consider disjunctive rules.
%
%
% For some $\FormatEntitySet{X} \in \{\Preds, \Funs, \Constants, \Vars, \Terms\}$ and a program $\prog$, we write $\EI{\FormatEntitySet{X}}{\prog}$ to denote the set of all elements of $\FormatEntitySet{X}$ that syntactically occur in $\prog$.
%
%
%
We assume that rules are \emph{safe}; that is, every variable in a rule occurs in some $T_i$ with $1 \leq i \leq m$.
%In other words, all variables that are used in a rule appear in some positive literal in the rule body.
We define $\HU(\prog)$ as the set of all (ground) terms that only feature function symbols and constants in \prog.
For a set of ground terms $T$, let $\Ground(\prog, T)$ be the set of ground rules containing $\Rule\Subs$ for every $\Rule \in \prog$ and every substitution $\Substitution$ from the the variables of $\Rule$ into $T$. 
Let $\Ground(\prog) = \Ground(\prog, \HU(\prog))$.
Note that $\Ground(\prog)$ might be infinite.
An \emph{answer set} of a program \prog is an answer set of $\Ground(\prog)$.
A ground program $\prog_g$ is a \emph{valid grounding} for a program \prog if 
$\prog_g$ and $\prog$ have the same answer sets.

\paragraph{Recap: Arithmetical/Analitical Hierarchy}

We view the \emph{arithmetical hierarchy} as classes of formal languages $\Sigma_i^0$ with $i \geq 1$ where $\Sigma_1^0$ is the class of all semi-decidable languages 
and $\Sigma_{i+1}^0$ results from $\Sigma_i^0$ by a Turing jump. The respective co-classes are denoted by $\Pi_i^0$.
We also consider the first level of the \emph{analytical hierarchy}, i.e. $\Sigma_1^1$ and $\Pi_1^1$, which is beyond the arithmetical hierarchy
\cite{RogersArithmeticalHierarchy}.
These classes are merely considered via reductions to and from languages contained in or hard for the respective classes
\cite{highly-undecidable-tilings}.

%% file: sections/3-consistency.tex
\section{Checking Consistency of ASP Programs}
\label{sec:consistency}

In this section, we prove that the problem of checking if a program admits an answer set is highly undecidable.
The upper bound follows from Lemma~\ref{lemma:membership-reduction} and Proposition~\ref{proposition:infinitely-visiting}; the lower bound from Lemma~\ref{lemma:hardness-reduction} and Proposition \ref{proposition:recurring-tiling}.
We include complete proofs for these lemmas in the technical appendix.

\begin{theorem}
\label{theorem:asp-consistency}
Deciding program consistency is $\Sigma_1^1$-complete.
\end{theorem}

Although Theorem~\ref{theorem:asp-consistency} has been proven before (see Corollary~5.12 in \cite{marek_stable_1994} and Theorem~5.9 in \cite{dantsin_complexity_2001}), 
we present complete proofs using (more) intuitive reductions 
that also lay our foundation for 
Section~\ref{sec:termination}.

%We point out the details since  first of all we think that our proofs are easy to follow and show very natural reductions  and  second we build heavly on the presented details in later sections.
%Intuitively, both reductions make use of an ``eventually'' quantification, meaning that something needs happen after an arbitrary but finite amount of steps.
%We find this to be a common characteristic of $\Sigma_1^1$-complete problems.

\subsection{An Upper Bound for ASP Consistency}
\label{sec:consistency-membership}

Our only goal in this subsection is to show that checking consistency is in $\Sigma_1^1$ by reduction to the following problem.

\begin{proposition}[{\cite[Corollary 6.2]{highly-undecidable-tilings}\footnote{The original result shows $\Pi_1^1$-completeness for the complement.}}]
\label{proposition:infinitely-visiting}
Checking if some run of a non-deterministic Turing machine on the empty word visits the start state infinitely many times is in $\Sigma_1^1$.
\end{proposition}
%In other words, at every step of the run, the machine needs to reach its start state again ``eventually'', i.e. after a finite amount of steps.

% We introduce one last preliminary definition before presenting our reduction:
For a program $P$ and an interpretation $I$, let $\textsf{Active}_{I}(P)$ be the set of all rules in $\textsf{Ground}(P)$ that are not satisfied by $I$.
If $I$ is finite, then so is $\textsf{Active}_{I}(P)$ and $\textsf{Active}_{I}$ is computable.

\begin{definition}
\label{definition:reduction-infinitely-visiting}
For a program $P$, let $\mathcal{M}_P$ be the non-deterministic machine that, regardless of the input, executes the following instructions:
\begin{enumerate}
\item Initialise an empty set $L_0$ of literals, and some counters $i := 0$ and $j := 0$.
\item If $L_i^+$ and $L_i^-$ are not disjoint, \emph{halt}.
\item If $L_i^+$ is an answer set of $P$, \emph{loop} on the start state.
\label{loop}
\item Initialise $L_{i+1} := L_i \cup H_r \cup \{\neg a \mid a \in B_r^-\}$ where $r$ is some non-deterministically chosen rule in $\textsf{Active}_{L_i^+}(P)$.\label{chooseActiveRules}
\item If $L_i$ satisfies all of the rules in $\textsf{Active}_{L_j^+}(P)$, then increment $j := j + 1$ and visit the start state once.
\item Increment $i := i + 1$ and go to Step~\ref{loop}.
\end{enumerate}
\end{definition}

\begin{lemma}
\label{lemma:membership-reduction}
A program $P$ is consistent iff some run of $\mathcal{M}_P$ on the empty word visits the start state infinitely many times.
\end{lemma}

Intuitively, a run of $\mathcal{M}_P$ attempts to produce an answer set for $P$; if successful, it visits the start state infinitely many times.
The answer set is materialised via the non-deterministic choices that instantiate 
% the sequence 
$L_1, L_2, \ldots$;
% of sets of literals 
see Step~\ref{chooseActiveRules}.

It is important to realize that the machine $\mathcal{M}_P$ only adds proven atoms in the sequence $L_1^+, L_2^+, \ldots$
To show this, consider some $k \geq 1$ and the ordering that maps the only atom in $L_i^+ \setminus L_{i-1}^+$ to $i$ for every $1 \leq i \leq k$.
Hence, if $L_k^+$ is a model of $P$, then $L_k^+$ also an answer set of $P$ and the run loops at the $k$-th iteration because of Step~\ref{loop}.
Otherwise, $\textsf{Active}_{L_k^+}(P)$ is non-empty and a rule from this set can be chosen in Step~\ref{chooseActiveRules}.

If the sequence $L_1, L_2, \ldots$ is infinite and $\mathcal{M}_P$ visits the start state infinitely many times during the considered run, then $\bigcup_{i \geq 1} L_i^+$ is a model of $P$.
This is because, for every $j \geq 1$, there is some $i \geq j$ such that $L_i^+$ satisfies all of the rules in $\textsf{Active}_{L_j^+}(P)$.
Therefore, since every atom in $\bigcup_{i \geq 1} L_i^+$ is proven, this interpretation is an answer set of $P$.

%
%We formally show that a program $P$ is consistent if and only if a run of $\mathcal{M}_P$ on the empty word visits the start state infinitely many times in Section~\ref{sec:appendix-consistency} of the appendix.

%Note that, for any given $i \geq 1$, the set $L_i$ only contains proven atoms$\textsf{Active}_{L_i^+}(P)$ in Step~\ref{chooseActiveRules} is non-empty since otherwise, if all rules are satisfied, $L_i^+$ is a model for $P$ but then (given that it does not feature $\bot$) it must either be an answer set or some atom in $L_i^+$ is not proven; by construction, the latter can only be the case if a atom occurs in both $L_i^+$ and $L_i^-$. All those cases are handled in earlier steps.
%The set $\textsf{Active}_{L_i^+}(P)$ in Step~\ref{chooseActiveRules} is non-empty since otherwise, if all rules are satisfied, $L_i^+$ is a model for $P$ but then (given that it does not feature $\bot$) it must either be an answer set or some atom in $L_i^+$ is not proven; by construction, the latter can only be the case if a atom occurs in both $L_i^+$ and $L_i^-$. All those cases are handled in earlier steps.

\subsection{A Lower Bound for ASP Consistency}
\label{sec:consistency-hardness}

Our only goal in this subsection is to show that checking consistency is $\Sigma_1^1$-hard by reduction from the following problem.
%We then formally discuss our reduction in Definition~\ref{definition:tiling-system} to conclude the proof of Theorem~\ref{trm:consistency-sigma-1-1-hard}.

\begin{definition}
\label{definition:tiling-system}
A \emph{tiling system} is a tuple $\langle T, \textit{HI}, \textit{VI}, t_0 \rangle$ where $T$ is a finite set of tiles, $\textit{HI}$ and $\textit{VI}$ are subsets of $T \times T$, and $t_0$ is a domino in $T$.
Such a tiling system admits a \emph{recurring solution} if there is a function $f : \Naturals \times \Naturals \to T$ such that:
\begin{enumerate}
\item For every $i, j \geq 0$, we have that $\langle f(i, j), f(i+1, j) \rangle \notin \textit{HI}$ and $\langle f(i, j), f(i, j+1) \rangle \notin \textit{VI}$.
\item There is an infinite subset $S$ of $\Naturals$ such that $f(0, j) = t_0$ for every $j \in S$.
\end{enumerate}
\end{definition}

\begin{proposition}[{\cite[Theorem 6.4]{highly-undecidable-tilings}\footnote{The original result shows $\Sigma_1^1$-completeness.}}]
\label{proposition:recurring-tiling}
Checking if a tiling system admits a recurring solution is $\Sigma_1^1$-hard.
\end{proposition}

Condition~2 in Defintion~\ref{definition:tiling-system} implies that, given any position in the first column, we will eventually find the special tile if we move upwards on the grid after a finite amount of steps.
%Having presented the target problem for our reduction, we are ready to formally define this function.

\begin{definition}
\label{definition:tiling-reduction}
For a \emph{tiling system} $\mathfrak{T} = \langle T, \textit{HI}, \textit{VI}, t_0 \rangle$, let $P_\mathfrak{T}$ be the program that contains the ground atom $\textit{Dom}(c_0)$ and all of the following rules:
\begin{align}
\textit{Dom}(s(X)) &\leftarrow \textit{Dom}(X) \label{rule:domain} \\
  \textit{Tile}_t(X, Y) &\leftarrow \textit{Dom}(X), \textit{Dom}(Y), \notag\\
  &\quad \{\lnot \textit{Tile}_{t'}(X, Y) \mid t' \in T \setminus \{t\}\} ~~ \forall t \in T \label{rule:tile-choice} \\
%\textit{Tile}_t(x, y) &\leftarrow \{\lnot \textit{Tile}_{t'}(x, y) \mid t' \in T \setminus \{t\}\} ~~~ \forall t \in T \\
  \leftarrow{} &\textit{Tile}_t(X, Y), \textit{Tile}_{t'}(s(X), Y) ~~ \forall \langle t, t'\rangle \in \textit{HI} \label{rule:h-incompatibility} \\
  \leftarrow{} &\textit{Tile}_t(X, Y), \textit{Tile}_{t'}(X, s(Y)) ~~ \forall \langle t, t'\rangle \in \textit{VI} \label{rule:v-incompatibility} \\
\textit{Below}_{t_0}(Y) &\leftarrow \textit{Tile}_{t_0}(c_0, s(Y)) \label{rule:below} \\
\textit{Below}_{t_0}(Y) &\leftarrow \textit{Below}_{t_0}(s(Y)) \label{rule:below-prop} \\
  \leftarrow{} &\textit{Dom}(Y), \lnot \textit{Below}_{t_0}(Y) \label{rule:under-t0}
\end{align}
%In the above, $x$ and $y$ are variables.
\end{definition}

\begin{lemma}
\label{lemma:hardness-reduction}
A tiling system $\mathfrak{T}$ admits a recurring solution iff the program $P_\mathfrak{T}$ is consistent
\end{lemma}

\newcommand{\Xsep}{0.56}

\newcommand{\Xa}{0}
\newcommand{\Xb}{1}
\newcommand{\Xc}{2}
\newcommand{\Xd}{3}
\newcommand{\Xe}{4}
\newcommand{\Xf}{5}
\newcommand{\Xg}{6}
\newcommand{\Xh}{7}
\renewcommand{\Xi}{8}

\newcommand{\Xj}{11}

\newcommand{\Ysep}{0.56}

\newcommand{\Ya}{0}
\newcommand{\Yb}{1}
\newcommand{\Yc}{2}
\newcommand{\Yd}{3}
\newcommand{\Ye}{4}
\newcommand{\Yf}{5}
\newcommand{\Yg}{6}
\newcommand{\Yh}{7}
\newcommand{\Yi}{8}

\definecolor{red}{HTML}{a93226}
\definecolor{blue}{HTML}{2471a3}

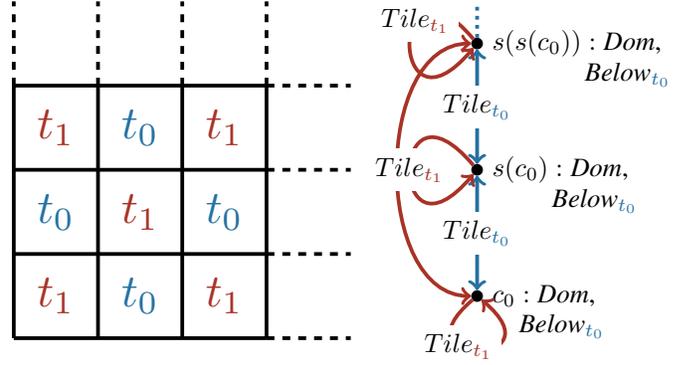
\begin{figure}~\\[-4.75em]
\begin{center}
\begin{tikzpicture}
\draw[black, line width=0.5mm] (\Xsep*\Xa, \Ysep*\Ya) -- (\Xsep*\Xg, \Ysep*\Ya);
\draw[dashed, line width=0.5mm] (\Xsep*\Xg, \Ysep*\Ya) -- (\Xsep*\Xi, \Ysep*\Ya);
\draw[black, line width=0.5mm] (\Xsep*\Xa, \Ysep*\Yc) -- (\Xsep*\Xg, \Ysep*\Yc);
\draw[dashed, line width=0.5mm] (\Xsep*\Xg, \Ysep*\Yc) -- (\Xsep*\Xi, \Ysep*\Yc);
\draw[black, line width=0.5mm] (\Xsep*\Xa, \Ysep*\Ye) -- (\Xsep*\Xg, \Ysep*\Ye);
\draw[dashed, line width=0.5mm] (\Xsep*\Xg, \Ysep*\Ye) -- (\Xsep*\Xi, \Ysep*\Ye);
\draw[black, line width=0.5mm] (\Xsep*\Xa, \Ysep*\Yg) -- (\Xsep*\Xg, \Ysep*\Yg);
\draw[dashed, line width=0.5mm] (\Xsep*\Xg, \Ysep*\Yg) -- (\Xsep*\Xi, \Ysep*\Yg);

\draw[black, line width=0.5mm] (\Xsep*\Xa, \Ysep*\Ya) -- (\Xsep*\Xa, \Ysep*\Yg);
\draw[dashed, line width=0.5mm] (\Xsep*\Xa, \Ysep*\Yg) -- (\Xsep*\Xa, \Ysep*\Yi);
\draw[black, line width=0.5mm] (\Xsep*\Xc, \Ysep*\Ya) -- (\Xsep*\Xc, \Ysep*\Yg);
\draw[dashed, line width=0.5mm] (\Xsep*\Xc, \Ysep*\Yg) -- (\Xsep*\Xc, \Ysep*\Yi);
\draw[black, line width=0.5mm] (\Xsep*\Xe, \Ysep*\Ya) -- (\Xsep*\Xe, \Ysep*\Yg);
\draw[dashed, line width=0.5mm] (\Xsep*\Xe, \Ysep*\Yg) -- (\Xsep*\Xe, \Ysep*\Yi);
\draw[black, line width=0.5mm] (\Xsep*\Xg, \Ysep*\Ya) -- (\Xsep*\Xg, \Ysep*\Yg);
\draw[dashed, line width=0.5mm] (\Xsep*\Xg, \Ysep*\Yg) -- (\Xsep*\Xg, \Ysep*\Yi);

\node[color=red] at (\Xsep*\Xb, \Ysep*\Yb) {\scalebox{1.75}{$t_1$}};
\node[color=blue] at (\Xsep*\Xb, \Ysep*\Yd) {\scalebox{1.75}{$t_0$}};
\node[color=red] at (\Xsep*\Xb, \Ysep*\Yf) {\scalebox{1.75}{$t_1$}};
\node[color=blue] at (\Xsep*\Xd, \Ysep*\Yb) {\scalebox{1.75}{$t_0$}};
\node[color=red] at (\Xsep*\Xd, \Ysep*\Yd) {\scalebox{1.75}{$t_1$}};
\node[color=blue] at (\Xsep*\Xd, \Ysep*\Yf) {\scalebox{1.75}{$t_0$}};
\node[color=red] at (\Xsep*\Xf, \Ysep*\Yb) {\scalebox{1.75}{$t_1$}};
\node[color=blue] at (\Xsep*\Xf, \Ysep*\Yd) {\scalebox{1.75}{$t_0$}};
\node[color=red] at (\Xsep*\Xf, \Ysep*\Yf) {\scalebox{1.75}{$t_1$}};

\node[label=right:{$c_0 : \textit{Dom}$,}, circle, draw, fill=black, inner sep=0pt, minimum width=4pt] (c0) at (\Xsep*\Xj,\Ysep*\Yb) {};
\node[label=right:{$\phantom{c_0 }\textit{Below}_{\color{blue}t_0}$}] (c0l) at (\Xsep*\Xj,\Ysep*\Yb-0.4) {};
\node[label=right:{$s(c_0) : \textit{Dom}$,}, circle, draw, fill=black, inner sep=0pt, minimum width=4pt] (sc0) at (\Xsep*\Xj,\Ysep*\Ye) {};
\node[label=right:{$\phantom{s(c_0) }\textit{Below}_{\color{blue}t_0}$}] (c0l) at (\Xsep*\Xj,\Ysep*\Ye-0.4) {};
\node[label=right:{$s(s(c_0)) : \textit{Dom}$,}, circle, draw, fill=black, inner sep=0pt, minimum width=4pt] (ssc0) at (\Xsep*\Xj,\Ysep*\Yh) {};
\node[label=right:{$\phantom{s(s(c_0)) }\textit{Below}_{\color{blue}t_0}$}] (c0l) at (\Xsep*\Xj,\Ysep*\Yh-0.4) {};
\node[fill=white, inner sep=0pt, minimum width=4pt] (sssc0) at (\Xsep*\Xj,\Ysep*\Yi) {};

\path[<->,blue,line width=0.5mm] (c0) edge node [fill=white, pos=0.5] {$\color{black}Tile_{\color{blue}t_0}$} (sc0);
\path[<->,blue,line width=0.5mm] (sc0) edge node [fill=white, pos=0.5] {$\color{black}Tile_{\color{blue}t_0}$} (ssc0);
\path[->, red,line width=0.5mm] (c0) edge [in=320,out=220,loop, distance=40] node [fill=white, pos=0.35] {$\color{black}Tile_{\color{red}t_1}$} (c0);
\path[<->, red,line width=0.5mm] (c0) edge [bend left=90] (ssc0);
\path[->, red,line width=0.5mm] (sc0) edge [in=230,out=130,loop, distance=50] node [fill=white, pos=0.5] {$\color{black}Tile_{\color{red}t_1}$} (sc0);
\path[->, red,line width=0.5mm] (ssc0) edge [in=230,out=130,loop, distance=50] node [fill=white, pos=0.35] {$\color{black}Tile_{\color{red}t_1}$} (ssc0);

\path[dotted, blue,line width=0.5mm] (ssc0) edge (sssc0);

\end{tikzpicture}
\end{center}~\\[-3.5em]
\caption{A Solution of $\mathfrak{X}$ and the Corresponding Answer Set of $P_{\mathfrak{X}}$}
\label{figure:recurring-tiling}
\end{figure}

Lemma~\ref{lemma:hardness-reduction} holds since each answer set of $P_\mathfrak{T}$ faithfully encodes a recurring solution of a tiling system $\mathfrak{T}$.
We clarify this brief intuition with an example.
\begin{example}
\label{example:periodic-tiling}
The tiling system $\mathfrak{X} = \langle \{\color{blue}t_0\color{black}, \color{red}t_1\color{black}\}, \textit{HI}, \textit{VI}, \color{blue}t_0\color{black} \rangle$ where $\textit{HI} = \textit{VI} = \{\langle \color{blue}t_0\color{black}, \color{blue}t_0\color{black} \rangle, \langle \color{red}t_1\color{black}, \color{red}t_1\color{black} \rangle\}$ admits two recurring solutions.
The program $P_\mathfrak{X}$ admits two answer sets; each of them encodes a solution of $\mathfrak{X}$.
One of these solutions and the corresponding answer set are depicted in Figure~\ref{figure:recurring-tiling}.
Note how the tile $\color{blue}t_0\color{black}$ covers the position $\langle 0, 1\rangle$ of the positive quadrant; this is encoded by the atom $\textit{Tile}_{\color{blue}t_0\color{black}}(c_0, s(c_0))$ in the answer set.
\end{example}

% \label{rule:domain}
% \label{rule:tile-choice}
% \label{rule:h-incompatibility} 
% \label{rule:v-incompatibility} \\
% \label{rule:below} \\
% \label{rule:below-prop} \\
% \label{rule:under-t0}

Intuitively, Rule~\ref{rule:domain} in Definition~\ref{definition:tiling-reduction} ensures that the domain of every answer set is countably infinite to provide enough space for a possible recurring solution.
Rule~\ref{rule:tile-choice} ensures that every position in the positive quadrant is covered by exactly one tile.
Constraints~\ref{rule:h-incompatibility} and \ref{rule:v-incompatibility} are violated if the horizontal and vertical incompatibilities are not satisfied, respectively.
Formulas~\ref{rule:below}, \ref{rule:below-prop}, and \ref{rule:under-t0} ensure that every position in the left column is under a position covered with the special tile that appears infinitely often in a valid recurring solution.

%% file: sections/4-termination.tex
\section{Frugal and Non-Proliferous Programs}\label{sec:termination}

In this section, we aim to develop a better understanding for why 
consistency has such a high level of undecidability.
One particularly hard (undecidable) case to check is the existentence of an infinite answer set,
so it is straightforward to restrict to ASP programs that only admit finite answer sets (they might not admit any or infinitely many of these).
Especially in cases, where we are not only interested in consistency but in enumerating all answer sets, it is also 
of interest that there is only finitely many of them.

\begin{definition}\label{def:frugalAndNonProliferous}
  A program is \emph{frugal} if it only admits finite answer sets; 
  it is \emph{non-proliferous} if it only admits finitely many finite answer sets (but arbitrarily many infinite ones).
\end{definition}

Not every frugal program is also non-proliferous.

\begin{example}
  The following ASP program admits infinitely many finite answer sets but no infinite one. 
  \begin{align*}
    \mathit{next}(Y, f(Y)) \leftarrow& \mathit{next}(X, Y), \neg \mathit{last}(Y).\\
    \mathit{last}(Y) \leftarrow& \mathit{next}(X, Y), \neg \mathit{next}(Y, f(Y)).\\
    \mathit{done} \leftarrow& \mathit{last}(Y). \quad \leftarrow \neg \mathit{done}. \quad \mathit{next}(c, d).
  \end{align*}

  Clearly, $\{ \mathit{next}(c, d), \mathit{last}(d), \mathit{done} \}$ is an answer set.
  Also, any finite chain of $\mathit{next}$ relations terminated by $\mathit{last}$ is an answer set. 
  However, an infinite $\mathit{next}$-chain is not an answer set 
  as it cannot contain any $\mathit{last}$ atom hence does not feature $\mathit{done}$ and therefore violates the constraint.
\end{example}

\subsection{Undecidability of these Notions}

Within the scope of this subsection, let $P$ be an arbitrary program.
Both of the above problems, i.e. $P$ being frugal or non-proliferous, are undecidable and not even semi-decidable. 
We start with the second problem since it is comparably ``easy''.

\begin{theorem}\label{trm:asp-finitely-many-finite-as-sigma-2-0}
  Deciding if $P$ is non-proliferous is $\Sigma_2^0$-complete.
\end{theorem}

The previous result follows directly from Lemmas~\ref{lem:asp-finitely-many-finite-as-sigma-2-0-mem} and \ref{lem:asp-finitely-many-finite-as-sigma-2-0-hard}.

\begin{lemma}\label{lem:asp-finitely-many-finite-as-sigma-2-0-mem}
  Deciding if $P$ is non-proliferous is in $\Sigma_2^0$.
\end{lemma}
\begin{proof}
  We show first ($\dagger$) that one can semi-decide for a given $n$ if a program has at least $n$ finite answer sets.
  This is possible by enumerating and checking all answer set candidates. Once the $n$-th answer set has been found, the procedure halts and accepts (otherwise it may run forever).

  The decision problem from the lemma can now be semi-decided with an oracle for ($\dagger$) as follows.
  Enumerate all naturals $n$ and check for each, if $P$ admits at least $n$ finite answer sets (with the oracle). 
  If yes, continue with $n+1$; 
  otherwise accept (since only finitely many finite answer sets exist).
\end{proof}

\begin{lemma}\label{lem:asp-finitely-many-finite-as-sigma-2-0-hard}
  Deciding if $P$ is non-proliferous is $\Sigma_2^0$-hard.
\end{lemma}
\begin{proof}
  Consider the universal halting problem, which is $\Pi_2^0$-hard, i.e. the check if a Turing machine halts on all inputs. 
  We construct $M'$ for a given TM $M$ that on input $n$ runs $M$ on all inputs of length at most $n$.
  We have that $M'$ terminates on infinitely many inputs if and only if $M$ halts on all inputs. 
  Hence, deciding if a TM halts on infinitely many inputs is $\Pi_2^0$-hard. 
  Therefore, the complement, i.e. deciding if a TM halts on only finitely many inputs, is $\Sigma_2^0$-hard. 
  We generate all (finite) inputs with an ASP program and ensure that the program has a finite answer set for a 
  generated input iff the TM halts on that input.\footnote{We show the machine simulation in the technical appendix.} %\ref{sec:appendix-termination}
  Therefore, deciding if an ASP program only admits finitely many finite answer sets is~$\Sigma_2^0$-hard.
\end{proof}

The Turing machine simulation utilizes a frugal program.
Therefore, checking if a program is non-proliferous remains $\Sigma_2^0$-hard for frugal programs.
Deciding if an ASP program is frugal on the other hand is way beyond $\Sigma_2^0$ and not even in the arithmetical hierarchy (just as checking consistency).

\begin{theorem}\label{trm:asp-only-finite-as-pi-1-1}
  Deciding if $P$ is frugal is $\Pi_1^1$-complete.
\end{theorem}
\begin{proof}[Proof Sketch.]
  For membership, we adjust the machine from Definition~\ref{definition:reduction-infinitely-visiting} 
  to halt instead of loop when it encounters a finite answer set. 
  Hardness follows from the same reduction as Lemma~\ref{lemma:hardness-reduction}. 
  To see that this holds, note that $P_\mathfrak{T}$ 
  either has an infinte answer set (i.e. is not frugal) if $\mathfrak{T}$ has a solution
  or has no answer set at all (i.e. is frugal) if $\mathfrak{T}$ has no solution.
\end{proof}

\subsection{Consistency becomes (only) Semi-Decidable}

If an ASP program is frugal, consistency is semi-decidable. 

\begin{theorem}\label{trm:consistencyForFrugalSemiDecidable}
  Consistency for frugal programs is in $\Sigma_1^0$.
\end{theorem}
\begin{proof}
  Enumerate all answer set candidates and check if they are answer sets. If there is an answer set, then there must be a finite one so the procedure terminates in this case. 
\end{proof}

Somewhat surprisingly, 
even for frugal and non-proliferous programs
%i.e. there are only finitely many answer sets, which are all finite, 
consistency remains undecidable.
The issue is that the maximum answer set size is still unknown.
%so we do not know when to terminate the (semi-decision) procedure.

\begin{theorem}\label{trm:consistencyFrugalAndNonProlifStillUndefidable}
  Consistency for frugal and non-proliferous programs is $\Sigma_1^0$-hard.
\end{theorem}
\begin{proof}
  We reduce from the halting problem with part of the program used for the TM simulation in Lemma~\ref{lem:asp-finitely-many-finite-as-sigma-2-0-hard}.
  We omit the part that generates all possible inputs.
  Instead, we encode the input word with ground atoms directly.
  The program admits a (single) finite answer set if the machine halts on its input.
  Otherwise, it does not admit any answer set.
  In any case, the program is both frugal and non-proliferous.
\end{proof}

The programs in the introduction have a finite bound on the size of their answer sets; they are frugal and non-proliferous.

%% file: sections/5-reasoning-procedures.tex
\section{Improved Reasoning Procedures}
\label{sec:reasoning-procedures}

In this section, we describe an approach for the 
computation of answer sets that builds upon a basic procedure 
for consistency checking.
For frugal programs, this is a semi-decision procedure.
However, unsatisfiable programs are often not detected as such.
Therefore, we improve the procedure by ignoring \emph{forbidden} atoms
that may never occur in any answer set.
While it is undecidable to check if an atom is forbidden, 
we give a proof-of-concept algorithm for a sufficient condition.
For some simplified but unsatisfiable versions of Examples~\ref{exp:artificial} and \ref{exp:wolf-goat-cabbage-intro},
we argue that the sufficient condition is powerful enough to detect the essential forbidden atoms.
This makes the enhanced semi-decision procedure detect them as unsatisfiable.
Later on, we also propose a procedure based on the forbidden atoms idea
to produce valid groundings that are finite more often compared to tranditional approaches.

% In particular for Examples~\ref{exp:artificial} and \ref{exp:wolf-goat-cabbage-intro}.

\subsection{Limits of Semi-Decision}

Even when $\Ground(\prog)$ is infinite, it is arguably not hard to come up with a semi-decision procedure for consistency 
when only considering frugal ASP programs $\prog$. 
This is the same as asking if an arbitrary program has a finite answer set.
We have shown semi-decidability in Theorem~\ref{trm:consistencyForFrugalSemiDecidable} 
and we can also achieve this by modifying $\mathcal{M}_P$ from Definition~\ref{definition:reduction-infinitely-visiting} 
such that it accepts when it encounters a finite answer set instead of entering an infinite loop.
While this machine resembles a lazy-grounding idea,
we may also describe a semi-decision procedure that incrementally enlarges a ground program. 
% The following idea already incorporates a straightforward improvement of rejecting (some)
% programs that are unsatisfiable under propositional semantics, thereby 
% going beyond a plain semi-decision procedure towards 
% covering more negative cases.

\begin{definition}
  Consider the procedure $\textsf{IsConsistent}(\cdot)$ that takes a program $P$ as input where $P_g = \textsf{Ground}(P)$:
\begin{enumerate}
\item Initialize $i := 1$ and $A_0 := \emptyset$.
\item Set $A_{i} := A_{i-1} \cup \bigcup \{ H_r \mid r \in P_g; B_r^+ \subseteq A_{i-1} \}$.\label{item:ground-step}
\item \emph{Reject} if $A_i = A_{i-1}$.
\item Set $P_i := \{ r \in P_g \mid B_r^+ \subseteq A_i \}$.
% \item \emph{Reject} if $P_i$ is not (propositionally) satisfiable.
\item \emph{Accept} if $P_i$ has an answer set $I$ with $\textsf{Active}_I(P) = \emptyset$.
\item Set $i := i + 1$ and go to Step~2.
\end{enumerate}
\end{definition}

\begin{proposition}\label{prop:isConsistentCorrect}
  Given some program $P$, the procedure $\textsf{IsConsistent}(P)$ accepts (and halts) if and only if $P$ has a finite answer set.
  % Additionally, if $\textsf{Ground}(P)$ is propositionally unsatisfiable, then $\textsf{IsConsistent}(P)$ rejects (and halts).
\end{proposition}
\begin{proof}
  If $P$ has a finite answer set $I$, then pick the smallest $i$ such that $A_{i} \supseteq I$. 
  % For every $1 \leq j \leq i$, $P_j$ is satisfiable by assigning ``true'' to exactly the ground atoms in $I$.
  Since every atom in $I$ is proven, the procedure does not reject up until reaching $i$. 
  Moreover, since $I$ is an answer set of $P$, it is an answer set of $P_i$; and $\textsf{Active}_I(P) = \emptyset$.
  Therefore, the procedure accepts in step $i$.

  If the procedure accepts, it does so for some $i$ and there is a (finite) answer set $I$ of $P_i$. Since $\textsf{Active}_I(P)$ is empty, all rules in $\textsf{Ground}(P)$ are satisfied by $I$. 
  Since $I$ is an answer set of $P_i$, all atoms in $I$ are proven in $P$. 
  Hence, $I$ is a finite answer set of $P$.
\end{proof}

Since $\textsf{IsConsistent}$ semi-decides consistency for frugal programs, it will necessarily yield a finite answer set for all examples from the introduction.
Still, inconsistent programs $\prog$ are rarely caught by $\textsf{IsConsistent}$, unless $\Ground(\prog)$ is finite. 
To illustrate this, we condense the encoding of Example~\ref{exp:wolf-goat-cabbage-intro} to a simple case of detecting redundancies. 

\begin{example}\label{exp:essential-wgc-different}
  % Consider the program $P'$ that results from $P$ in Example~\ref{exp:essential-wgc-same}
  % by replacing the last two rules with:
  % \begin{align*}
  %   &\textit{diff}(N, M) \leftarrow \textit{fct}(X, N), \textit{fct}(Y, M), \neg \textit{eq}(X, Y), \textit{lt}(N, M).\\
  %   &\textit{redundant} \leftarrow \textit{fct}(X, N), \textit{fct}(Y, M), \textit{lt}(N, M), \neg \textit{diff}(N, M).
  % \end{align*}

  Consider the following program $P$.
  \begin{align*}
    &\textit{fct}(a, 0). \qquad \textit{eq}(X, X) \leftarrow \textit{fct}(X, N). \qquad \leftarrow \textit{redundant}. \\
    &\textit{lt}(N, s(N)) \leftarrow \textit{fct}(X, s(N)).\\
    &\textit{lt}(N, N') \leftarrow \textit{lt}(N, M), \textit{lt}(M, N').\\
    &\textit{fct}(b, s(N)) \leftarrow \textit{fct}(a, N). \qquad \textit{fct}(a, s(N)) \leftarrow \textit{fct}(b, N).\\
    &\textit{diff}(N, M) \leftarrow \textit{fct}(X, N), \textit{fct}(Y, M), \neg \textit{eq}(X, Y), \textit{lt}(N, M).\\
    &\textit{redundant} \leftarrow \textit{fct}(X, N), \textit{fct}(Y, M), \textit{lt}(N, M), \neg \textit{diff}(N, M).
  \end{align*}

  Intuitively, a timeline is constructed that always flips fact (\textit{fct}) $a$ to $b$ and vice versa in each step. 
  We forbid redundancies, e.g. we do not want \textit{fct} $a$ in time steps say $0$ and $2$.
  This is impossible and we will always be forced to derive $\textit{redundant}$ at some point because some of the $\textit{diff}(N,M)$ atoms cannot be proven.
  However, the procedure $\textsf{IsConsistent}$ simply does not terminate here.
\end{example}
Note that we ran the example with Alpha, but encountered a stack overflow. This may indicate that Alpha introduces too many ground atoms.
Also, (i)DLV as well as gringo/clingo do not (seem to) terminate.
The \textsf{IsConsistent} check also runs into similar problems for a slight variation of Example~\ref{exp:artificial}.
\begin{example}\label{exp:artificial-no-answer-set}
  Consider the following extension of Example~\ref{exp:artificial}.
  \begin{align*}
    &\mathit{r}(a, b). \qquad \mathit{stop}(Y) \leftarrow \mathit{r}(X, Y). \qquad \leftarrow \mathit{r}(b, f(b)).\\
    &\mathit{r}(Y, f(Y)) \leftarrow \mathit{r}(X, Y), \neg \mathit{stop}(X).
  \end{align*}
  The program does not have any answer set.
  Furthermore, the $\textsf{IsConsistent}$ check does not terminate.  
  %because $\textsf{Ground}(P)$ is satisfiable under propositional semantics.
  % The minimal propositional model is $\{ \mathit{r}(a, b), \mathit{stop}(b), \mathit{stop}(a) \}$.
\end{example}
To our surprise, Alpha captures this as unsatisfiable.
Still, (i)DLV and gringo/clingo do not (seem to) terminate. 

\subsection{Ignoring Forbidden Atoms}

We aim to extend $\textsf{IsConsistent}$ further to capture the previous examples as unsatisfiable.
We adjust the assignment of $A_i$ in Item~\ref{item:ground-step} as follows; keeping Proposition~\ref{prop:isConsistentCorrect} intact.
\begin{enumerate}
  \item[2.] Set $A_{i} := A_{i-1} \cup \{ a \mid a \text{ not \emph{forbidden} in } P \text{ and } \{a\} = H_r \text{ for some } r \in \textsf{Ground}(P) \text{ with } B_r^+ \subseteq A_{i-1} \}$.
\end{enumerate}

To keep the $A_{i}$ (and thus $P_{i}$) small, we only consider atoms that are not forbidden in $P$. 
Formally, an atom $a$ is \emph{forbidden} (in a program $P$) if $a$ does not occur in any answer set of $P$. 
Intuitively, a forbidden atom will necessarily lead to a contradiction or it will be impossible to show that it is proven.
Unfortunately, it is undecidable to check if an atom is forbidden in the formal sense, essentially because entailment of ground atoms over ASP programs is undecidable.

\begin{proposition}
  It is undecidable if an atom is forbidden.
\end{proposition}
\begin{proof}
  We reuse the Turing machine simulation from Theorem~\ref{trm:consistencyFrugalAndNonProlifStillUndefidable}
  to show a reduction from the complement of the halting problem.
  If the machine halts, the simulation has a single answer set with the atom $\mathit{Halt}$. 
  Otherwise, the simulation does not admit an answer set. 
  Hence, the machine does not halt iff $\mathit{Halt}$ is forbidden.
\end{proof}

We introduce some auxiliary definitions with the aim of giving a (rather tight) sufficient condition for finding forbidden atoms.
Given a rule $r$ in $P$ and any two interpretations $L^+$ and $L^-$, we define the following. 

\begin{itemize}
  \item $r^+(L^+, L^-)$ as the minimal interpretation that contains the single atom in $H_r\sigma$ (unless $\vert H_r \vert = 0$) for every substitution $\sigma$ with $B_r^+\sigma \subseteq L^+$ and $B_r^-\sigma \subseteq L^-$; 
    and, if $\lvert B_r^- \rvert = 1$, also contains $B_r^-\sigma$ for every substitution $\sigma$ with $B_r^+\sigma \subseteq L^+$ and $H_r\sigma \subseteq L^-$.
  \item $r^-(L^+, L^-)$ is $\emptyset$ if $\lvert B_r^+ \rvert \neq 1$; otherwise it is defined as the minimal interpretation that contains $B_r^+\sigma$ for every substitution $\sigma$ with $B_r^-\sigma \subseteq L^-$ and $H_r\sigma \subseteq L^-$
\end{itemize}

We define a way of applying rules ``in reverse'' here and restrict these cases to $\lvert B_r^- \rvert = 1$ and $\lvert B_r^+ \rvert = 1$, respectively. 
In principle one could relax this by considering every possible choice for the atoms from $B_r^-$ or $B_r^+$. 
% , so we disallow this for simplicity. 
It remains for practical evaluations to determine if this is a good trade-off between simplicity and generality.
For a program $P$ and two interpretations $L^+$ and $L^-$, we define the following.
For any interpretation $L$, let the \emph{term-atoms} $\textit{TA}^P(L)$ be the set of all (ground) atoms with predicates and ground terms from $P$, terms in $L$, and arbitrary constants.
Also, for each sign $s \in \{+,-\}$:

\begin{itemize}
  \item $P_0^s(L^+, L^-)$ as $L^s$, and
  \item for every $i \geq 0$, $P_{i+1}^s(L^+, L^-)$ as the minimal interpretation that contains $P_i^s(L^+, L^-)$ and, for every $r \in P$, 
    $r^s(P_i^+(L^+, L^-), P_i^-(L^+, L^-)) \cap \textit{TA}^P(L^+ \cup L^-)$
  \item $P_\infty^s(L^+, L^-)$ as $\bigcup_{i \geq 0} P_i^s(L^+, L^-)$
\end{itemize}

We intersect with $\textit{TA}^P(L^+ \cup L^-)$ only to keep $P_j^s(\dots)$ finite for all $j$ (and hence for $\infty$); almost arbitrary extensions are possible here. 
Intuitively, for sets of atoms that must be true $L^+$ and must be false $L^-$, $P_\infty^+$ and $P_\infty^-$ close these sets under 
certain (not all, as there might be infinitely many) inferences to obtain larger sets of atoms that must be true or false, respectively.
By definition, we can show via induction that $P_\infty^s$ only makes sound inferences of atoms. 

\begin{lemma}\label{lem:PClosureContainedInEveryAS}
  For a program $P$ any two interpretations $L^+$ and $L^-$ and any answer set $I$ of $P$; 
  if $L^+ \subseteq I$ and $L^- \cap I = \emptyset$,
  then $P_\infty^+(L^+, L^-) \subseteq I$ and $P_\infty^-(L^+, L^-) \cap I = \emptyset$.
\end{lemma}

To check if an atom $a$ is forbidden in a program $P$, we backtrack atoms that must be true ($L^+$) and must be false ($L^-$) to be able to prove $a$. 
To this aim, we describe a procedure with the help of some auxiliary definitions.
We say that an atom \emph{has support} in $P$ for interpretations $L^+$ and $L^-$ if 
there is a rule $r \in P$ and a substitution $\sigma$ with $H_r\sigma = \{a\}$, $B_r^+\sigma \subseteq L^+$ and $B_r^-\sigma \subseteq L^-$. 
Intuitively, having support is almost like being proven but the set of non-derived atoms is given explicitly using $L^-$.
For a rule $r \in P$ and a substitution $\sigma$, an \emph{$r$-extension} of $\sigma$ is a substitution $\sigma'$ that agrees with $\sigma$ on variables from $H_r$
and additionally,
for each variable $X$ in $r$ that occurs in $B_r^+$ but not in $H_r$, 
if $X$ occurs in a position that can only feature constants, $\sigma'(X)$ is one of these constants;
otherwise, i.e. if functional terms may occur in the position of $X$, $\sigma'(X)$ is a fresh constant.\footnote{
  We can perform static analysis on 
  the positive part of $P$ to obtain the possible constants
  for each position in $P$ and also to determine if a function symbol might occur in a position.
} 
To check if $a$ is forbidden, we call $\textsf{IsForbidden}(P, \{a\}, \emptyset)$
from Algorithm~\ref{alg:anyForbidden}.
\begin{algorithm}[t]
  \caption{\textsf{IsForbidden}}\label{alg:anyForbidden}
  \begin{algorithmic}[1]
  \REQUIRE program $P$, interpretations $L^+$ and $L^-$
  \ENSURE some atom in $L^+$ is forbidden or $L^+ \cap L^- \neq \emptyset$
  \STATE $L^+ \gets P_\infty^+(L^+, L^-)$ and $L^- \gets P_\infty^-(L^+, L^-)$ \label{algoLine:PClosure}
  \IF{$L^+ \cap L^- \neq \emptyset$} 
    \RETURN \TRUE \label{algoLine:true}
  \ENDIF
  \STATE someFrbdn $\gets$ \FALSE
  % \STATE \COMMENT {For each atom $a$ that still needs to be proven...}
  \FORALL{$a \in L^+$ without support} \label{algoLine:loopAtoms}
    \STATE aFrbdn $\gets$ \TRUE 
    % \STATE \COMMENT {... check each way in which $a$ could be proven.}
    \FORALL{$r \in P$ \textbf{and} substitution $\sigma$ \textbf{and} $g$ mapping all fresh constant to arbitrary terms except such constants \textbf{with} $H_r\sigma = \{g(a)\}$} \label{algoLine:loopRules}
      \IF{$g(a)$ features terms not in $L^+ \cup L^-$} \label{algoLine:ifStillFresh}
        % \STATE \COMMENT {If $a$ can be proven from some unknown functional terms, assume it is not forbidden.}
        \STATE aFrbdn $\gets$ \FALSE
        \STATE \textbf{break}
      \ENDIF
      \STATE $K^+ \gets g(L^+)$ and $K^- \gets g(L^-)$ \label{algoLine:gUpdate}
      \FORALL{$r$-extensions $\sigma'$ of $\sigma$} \label{algoLine:freshConst}
        \STATE $J^+ \gets K^+ \cup (B_r^+\sigma' \cap \textit{TA}^P(K^+ \cup K^-))$ and $J^- \gets K^- \cup (B_r^-\sigma' \cap \textit{TA}^P(K^+ \cup K^-))$ \label{algoLine:extensionUpdate}
        % \STATE \COMMENT {If $a$ can be proven without involving forbidden facts or a contradiction, $a$ is not forbidden.}
        \STATE aFrbdn $\gets$ aFrbdn $\land$ $\textsf{IsForbidden}(P, J^+, J^-)$
      \ENDFOR
    \ENDFOR
    \STATE someFrbdn $\gets$ someFrbdn $\lor$ aFrbdn
  \ENDFOR
  \RETURN someFrbdn
  \end{algorithmic}
\end{algorithm}
The general idea of the procedure is as follows. 
We first check if $L^+$ and $L^-$ contradict each other. When we reach this base case, we know that our initial atom $a$ must be forbidden. 
Otherwise, we check if there is an unproven atom $a$ left in line~\ref{algoLine:loopAtoms}. 
If so, we check all ways in which $a$ could be proven in line~\ref{algoLine:loopRules}.
If $a$ can potentially proven with some unknown function symbols, we assume $a$ not to be forbidden in line~\ref{algoLine:ifStillFresh}.
Otherwise, we perform recursive calls to \textsf{IsForbidden} within the loop in line~\ref{algoLine:freshConst} 
such that $a$ is marked as not forbidden if at least one recursive call does not involve forbidden atoms or a contradiction.
This means, $a$ might be provable.
By a recursive analysis of the algorithm, one can verify correctness 
with the help of Lemma~\ref{lem:PClosureContainedInEveryAS}.

\begin{theorem}\label{trm:forbiddenCheckCorrect} 
  If the output of $\textsf{IsForbidden}(P, \{a\}, \emptyset)$ is true, then the atom $a$ is forbidden in $P$.
\end{theorem}

We are able to show for 
% Examples~\ref{exp:essential-wgc-same} and \ref{exp:essential-wgc-different} 
Example~\ref{exp:essential-wgc-different} 
that 
$\textit{fct}(a, s(s(0)))$ is forbidden
since the atom $\textit{diff}(0, s(s(0)))$ cannot possibly have support.
This would require $\neg \textit{eq}(a, a)$, which contradicts $\textit{eq}(a, a)$.
% For Example~\ref{exp:essential-wgc-same} this is straightforward since we will get an inconsistency almost immediately using $\textit{same}(0, s(s(0)))$.
% For Example~\ref{exp:essential-wgc-different}, this comes down to showing that $\textit{diff}(0, s(s(0)))$ cannot possibly have support 
For Example~\ref{exp:artificial-no-answer-set}, it is key to notice that $r(f(b), f(f(b)))$ is forbidden.
% which we can also show with Algorithm~\ref{alg:anyForbidden}.

\begin{example}
  We show how Algorithm~\ref{alg:anyForbidden} verifies that $r(f(b), f(f(b)))$ is forbidden in Example~\ref{exp:artificial-no-answer-set}.
  \begin{itemize}
    \item Initialize $L^+$ with $r(f(b), f(f(b)))$ and $L^-$ with $\emptyset$.
    \item In line~\ref{algoLine:PClosure}, 
      $P_\infty^-(L^+, L^-) = \{ r(b, f(b)) \}$;
      $P_\infty^+(L^+, L^-) = \{ r(f(b), f(f(b))), r(a, b), stop(b), stop(f(f(b))) \}$.
    \item In the loop in line~\ref{algoLine:loopAtoms}, pick $r(f(b), f(f(b)))$.
    \item In the loop in line~\ref{algoLine:loopRules}, 
      there is only one choice with $r$ as the last rule, $g$ the identity, 
      and $\sigma$ mapping $Y$ to $f(b)$. 
    \item Since $g$ is the identity,
      the condition in line~\ref{algoLine:ifStillFresh} is false.
    \item $K^+$ and $K^-$ are $P_\infty^+(L^+, L^-)$ and $P_\infty^-(L^+, L^-)$.
    \item For the $r$-extension of $\sigma$ in line~\ref{algoLine:freshConst}, we pick $\sigma'$ with 
      $\sigma'(X) = c_x$ where $c_x$ is a fresh constant.
      (The first position of $r$ may feature function symbols.)
    \item In line~\ref{algoLine:extensionUpdate}, $J^+$ is set to $K^+$ extended with $r(c_x, f(b))$ 
      and $J^-$ is set to $K^-$ extended with $stop(c_x)$.
    \item In the recursive call, we initialize $L^+$ and $L^-$ with $J^+$ and $J^-$ from outside the call.
    \item In line~\ref{algoLine:PClosure}, the atom closures are 
      $P_\infty^-(L^+, L^-) = L^-$ and
      $P_\infty^+(L^+, L^-) = L^+ \cup \{ stop(f(b)) \}$.
    \item In the loop in line~\ref{algoLine:loopAtoms}, pick $r(c_x, f(b))$.
    \item In the loop in line~\ref{algoLine:loopRules}, 
      there is only one choice with $r$ as the last rule, $g$ mapping $c_x$ to $b$, 
      and $\sigma$ mapping $Y$ to $b$. 
    \item In line~\ref{algoLine:gUpdate}, obtain $K^+$ and $K^-$ 
      from $P_\infty^+(L^+, L^-)$ and $P_\infty^-(L^+, L^-)$ by replacing $r(c_x, f(b))$ by $r(b, f(b))$ 
      and $stop(c_x)$ by $stop(b)$, respectively.
    \item For the $r$-extension of $\sigma$ in line~\ref{algoLine:freshConst}, we pick $\sigma'$ with 
      $\sigma'(X) = c_x$ where $c_x$ is again a fresh constant.
    \item In line~\ref{algoLine:extensionUpdate}, $J^+$ is set to $K^+$ extended with $r(c_x, b)$ 
      and $J^-$ is set to $K^-$ extended with $stop(c_x)$.
    \item In the recursive call, we initialize $L^+$ and $L^-$ with $J^+$ and $J^-$ from outside the call.
    \item In line~\ref{algoLine:PClosure}, the atom closures are 
      $P_\infty^-(L^+, L^-) = L^-$ and
      $P_\infty^+(L^+, L^-) = L^+ \cup \{ stop(b) \}$.
    \item We return true in line~\ref{algoLine:true} since 
      $stop(b) \in L^+ \cap L^-$.
  \end{itemize}
\end{example}

% \footnote{Details for $r(f(b), f(f(b)))$ and Example~\ref{exp:artificial-no-answer-set} in the appendix.} 

Based on these insights, we conclude that 
the \textsf{IsConsistent} check ignoring forbidden atoms
according to the \textsf{IsForbidden} check 
captures Examples~\ref{exp:essential-wgc-different} and \ref{exp:artificial-no-answer-set}
as unsatisfiable.

\subsection{Towards Finite Valid Groundings}

Avoiding forbidden atoms does not only improve the $\textsf{IsConsistent}$ procedure but
can reduce grounding efforts.

\begin{definition}
  Define $\textsf{GroundNotForbidden}(\cdot)$ that takes a program $P$ as input and executes the following instructions:
  \begin{enumerate}
  \item Initialize $i := 1$, $A_0 := \emptyset$, and $P_g := \emptyset$.
  \item Initialize $A_i := A_{i-1}$ and for each $r \in \textsf{Ground}(P)$ with $B_r^+ \subseteq A_{i-1}$, do the following. 
    If all atoms in $H_r$ are forbidden in $P$, add $\leftarrow B_r$ to $P_g$.
    Otherwise, add $r$ to $P_g$ and add the single atom in $H_r$ to $A_i$.
  \item \emph{Stop} if $A_i = A_{i-1}$; else set $i := i + 1$ and go to Step~2.
  \end{enumerate}
  The output of the procedure is $P_g$.
\end{definition}

Observe that for any answer set $I$ of a program $P$, $P_g$ still proves all atoms in $I$ 
and the rules $P_g \setminus \textsf{Ground}(P)$ must be satisfied by $I$.
Vice versa, all atoms in any answer set $I'$ for $P_g$ are also proven in 
$\textsf{Ground}(P)$ and $I'$ satisfies all rules in $\textsf{Ground}(P)$ 
as otherwise a contradiction would follow from one of the constraints in $P_g \setminus \textsf{Ground}(P)$.

\begin{theorem}\label{trm:groundWithoutForbiddenValid}
  For a program $P$, $\textsf{GroundNotForbidden}(P)$ is a valid grounding,
  i.e. $I$ is an answer set of $P$ iff $I$ is an answer set of $\textsf{GroundNotForbidden}(P)$.
\end{theorem}

Interestingly, the procedure always yields a finite ground program $P_g$ for frugal and non-proliferous programs.

\begin{proposition}\label{prop:groundWithoutForbiddenFiniteForFrugalNonProlif}
  For a frugal and non-proliferous program $P$, $\textsf{GroundNotForbidden}(P)$ is finite.
\end{proposition}
\begin{proof}
  By Definition~\ref{def:frugalAndNonProliferous}, $P$ has only finitely many answer sets and all of them are finite.
  Hence, there is only a finite number of atoms in all answer sets of $P$
  and all other atoms are forbidden. 
  Therefore, $\textsf{GroundNotForbidden}(P)$ 
  only takes a finite number of steps to compute.
\end{proof}

This result might be surprising but it is less so once we realize that this only holds 
because the definition of \textsf{GroundNotForbidden} 
assumes that we can decide if an atom is forbidden. 
So \textsf{GroundNotForbidden} is actually not computable.
In practice however, any sufficient check for forbiddenness (like \textsf{IsForbidden}) can be used in the procedure 
to make it computable
without sacrificing the validity of $P_g$, as then $P_g$ will only contain more rules from $\textsf{Ground}(P)$
and Theorem~\ref{trm:groundWithoutForbiddenValid} still holds.
This allows us to compute finite valid groundings for all of our Examples~\ref{exp:artificial}, \ref{exp:wolf-goat-cabbage-intro}, \ref{exp:essential-wgc-different}, and \ref{exp:artificial-no-answer-set}.

%% file: sections/6-conclusion.tex
\section{Conclusion}

In this work, we have been undergoing an in-depth reconsideration 
of undecidability for ASP consistency.
We have shown intuitive reductions from and to similarly hard problems, 
which can be of interest even outside of the ASP community.
We also considered two characteristics of ASP programs that 
can make reasoning hard, that is having infinite answer sets or infinitely many answer sets. 
We identified \emph{frugal} and \emph{non-proliferous} programs as a desirable class of programs 
that at least ensures semi-decidability of reasoning. 
To cover more negative cases with a semi-decision procedure, we 
ignore \emph{forbidden} atoms and show a proof-of-concept algorithm 
implementing a sufficient condition that captures our main examples. 
Furthermore, we may leverage forbidden atoms also to compute a valid grounding. 
In principle, \textsf{GroundNotForbidden} yields finite valid groundings for all frugal and non-proliferous programs 
if we are able to ignore all forbidden atoms. 
Note that, while we considered non-disjunctive programs for simplicity, 
our results can be extended towards disjunctions as well. 
This requires careful reconsiderations for Algorithm~\ref{alg:anyForbidden}
but is almost immediate for all other results.

\paragraph{Future Research Directions and Outlook.}
We hope this work will reopen the discussion about function symbols and how we can design smart techniques to avoid non-terminating grounding procedures. 
%Indeed, despite the fact that function symbols might cause infinite groundings, oftentimes we can avoid sources of infity becoming active.
%
We expect that the ASP community would benefit from incorporating recent research on termination conditions and updating their grounding strategies accordingly. After all, the ASP language is first-order based and we are convinced that function symbols are a key ingredient for elegant and convenient modeling of real-world scenarios.

An obvious future work is an efficient implementation 
of sufficient checks for forbidden atoms, requiring a good tradeoff between generality and performance. 
Our proposed procedure can function as a reference for implementation. 
Grounding procedures ignoring forbidden atoms can then be evaluated for existing ASP solvers. 
Even in lazy-grounding, termination of solvers can be improved by ignoring forbidden atoms.
We think this idea is promising for improving existing reasoners like Alpha, gringo/clingo, and (i)DLV.

%% file: sections/a-examples.tex
\section{Additional material for Section~\ref{section:introduction}}

We give the full modelling of the Wolf Goat Cabbage Puzzle
inspired by lecture slides from Jean-François Baget.
We also provide a Github repository containing all examples used in the paper.\footnote{\url{https://github.com/monsterkrampe/ASP-Termination-Examples}}

\begin{lstlisting}[caption="Wolf Goat Cabbage Puzzle in ASP",label={lst:wolf-goat-cabbage, language=dflat}]
steps(0..100).

bank(east). 
bank(west).
opposite(east,west). 
opposite(west,east).
passenger(wolf). 
passenger(goat). 
passenger(cabbage).
position(wolf,west,0). 
position(goat,west,0). 
position(cabbage,west,0).
position(farmer,west,0).
eats(wolf,goat). 
eats(goat,cabbage).

win(N) :- position(wolf,east,N),
  position(goat,east,N),
  position(cabbage,east,N).
winEnd :- win(N).
:- not winEnd.
lose :- position(X,B,N), 
  position(Y,B,N), eats(X,Y), 
  position(farmer,C,N), opposite(B,C).
:- lose.

goAlone(N) :- position(farmer,B,N), 
  not takeSome(N), not win(N).
takeSome(N) :- position(farmer,B,N),
  passenger(Y), position(Y,B,N),
  not goAlone(N), not win(N).

transport(X,N) :- takeSome(N),
  position(X,B,N), position(farmer,B,N),
  passenger(X), not othertransport(X,N).
othertransport(X,N) :- position(X,B,N),
  transport(Y,N), X != Y.

position(X,C,N+1) :- transport(X,N), 
  position(X,B,N), opposite(B,C), steps(N+1).
position(X,B,N+1) :- position(X,B,N), 
  passenger(X), not transport(X,N), 
  not win(N), steps(N+1).
position(farmer,C,N+1) :- 
  position(farmer,B,N), opposite(B,C), 
  not win(N), steps(N+1).

change(N,M) :- position(X,B,N), 
  position(X,C,M), opposite(B,C), N < M.
redundant :- position(X,B,N), 
  position(X,B,M), N < M, not change(N,M).
:- redundant.
\end{lstlisting}

%% file: sections/b-proofs-consistency.tex
\section{Proofs for Section~\ref{sec:consistency}}\label{sec:appendix-consistency}

\subsection{Proof of Lemma~\ref{lemma:membership-reduction}}

In this subsection we prove Lemmas~\ref{lemma:reduction-membership-1} and \ref{lemma:reduction-membership-2}, which directly imply Lemma~\ref{lemma:membership-reduction}.

\begin{lemma}
\label{lemma:reduction-membership-1}
If a program $P$ is consistent, then a run of $\mathcal{M}_P$ on the empty word visits the start state infinitely many times.
\end{lemma}
\begin{proof}
  Let $I$ be an answer set of $P$, i.e. $I$ is an answer set of $\textsf{Ground}(P)$.
  Hence, there is an ordering $\varphi$ on the atoms in $I$.
  Every atom $a \in I$ is proven by some ground rule $r \in \textsf{Ground}(P)$ using only atoms that are ancestors of $a$ with respect to $\varphi$. 
  There is a non-deterministic computation branch of $\mathcal{M}_P$ that chooses these rules in an order that respects $\varphi$. 
  If $I$ is finite, there is an $i \geq 0$ with $L_i^+ = I$ and the machine $\mathcal{M}_P$ visits its start state infinitely many times. 
  If $I$ is infinite, we still know that eventually $\textsf{Active}_{I}(P) = \emptyset$. 
  Hence, for every $j \geq 0$, there is an $i \geq j$ such that 
  $\textsf{Active}_{L_j^+}(P)$ and $\textsf{Active}_{L_i^+}(P)$ are disjoint. 
  Also in this case, the machine $\mathcal{M}_P$ visits its start state infinitely many times.
\end{proof}

\begin{lemma}
\label{lemma:reduction-membership-2}
  A program $P$ is consistent if a run of $\mathcal{M}_P$ on the empy word visits the start state infinitely many times.
\end{lemma}
\begin{proof}
  If $\mathcal{M}_P$ has a non-deterministic computation path that visits its start state infinitely many times, then there are two cases to consider. 
  First, $L_i^+$ is an answer set of $P$ for some $i \geq 0$. Then the claim follows.
  Second, $j$ grows towards infinity, i.e. for each $j \geq 0$, there is a $k \geq j$ such that 
  $\textsf{Active}_{L_j^+}(P)$ and $\textsf{Active}_{L_k^+}(P)$ are disjoint. This also requires that for every $i \geq 0$, $L_i^+$ and $L_i^-$ are disjoint.
  Now on one hand by construction we get that for every $i \geq 0$, every ground atom in $L_i^+$ is proven. (The construction directly yields the necessary ordering since each step adds one additional atom to $L_i^+$.)
  And on the other hand, for every rule in $r \in \textsf{Ground}(P)$, there is an $i \geq 0$ such that $L_{i'}^+$ satisfies $r$ for every $i' \geq i$. 
  Hence, for $I = \bigcup_{i \geq 0} L_i^+$, we have that $\textsf{Active}_{I}(P) = \emptyset$ (and still that every atom in $I$ is proven). 
  Therefore, $I$ is an answer set of $P$. 
\end{proof}

\subsection{Proof of Lemma~\ref{lemma:hardness-reduction}}

In this subsection we prove Lemmas~\ref{lemma:reduction-hardness-tiling-to-asp} and \ref{lemma:reduction-hardness-asp-to-tiling}, which directly imply Lemma~\ref{lemma:hardness-reduction}.

\begin{lemma}
\label{lemma:reduction-hardness-tiling-to-asp}
If a tiling system $\mathfrak{T}$ admits a recurring solution, then the program $P_\mathfrak{T}$ admits an answer set.
\end{lemma}
\begin{proof}
Consider a tiling system $\mathfrak{T} = \langle T, \textit{HI}, \textit{VI}, t_0 \rangle$ that does admit a recurring solution.
That is, there is function $f : \mathbb{N} \times \mathbb{N} \to T$ such as the one described in Definition~\ref{definition:tiling-system}.
To prove the lemma we verify that the following set of ground atoms is an answer set of $P_\mathfrak{T}$:
\begin{align*}
I = \{&\textit{Dom}(s^i(c_0)), \textit{Below}_{t_0}(s^i(c_0)) \mid i \geq 0\} \cup~ \\
\{&\textit{Tile}_{f(i, j)}(s^i(c_0), s^j(c_0)) \mid i, j \geq 0\}
\end{align*}

We first show that $I$ satisfies every rule in $\textsf{Ground}(P_\mathfrak{T})$ with a case-by-case analysis:
\begin{itemize}
\item The atom $\textit{Dom}(c_0)$, resulting from a rule with empty body in $\textsf{Ground}(P_\mathfrak{T})$, must be in $I$.
\item Rules of Type~\eqref{rule:domain} are satisfied since $\textit{Dom}(s^i(c_0)) \in I$ for every $i \geq 1$.
\item Rules of Type~\eqref{rule:tile-choice} are satisfied because $f$ is a total function and $\{s^i(c_0) \mid i \geq 0\}$ is the set of all terms in $I$.
\item Rules of Type~\eqref{rule:h-incompatibility} and \eqref{rule:v-incompatibility} are satisfied because $\langle f(i, j), f(i+1, j) \rangle \notin \textit{HI}$ and $\langle f(i, j), f(i, j+1) \rangle \notin \textit{VI}$ for every $i, j \geq 0$ by Definition~\ref{definition:tiling-system}.
\item Rules of Type~\eqref{rule:below}, \eqref{rule:below-prop} and \eqref{rule:under-t0} are satisfied since $\textit{Below}_{t_0}(s^i(c_0)) \in I$ for every $i \geq 0$.
\end{itemize}

Moreover, we argue that every atom $a \in I$ is proven (obtaining an appropriate ordering of atoms in $I$ is straightforward): 
\begin{itemize}
\item $\textit{Dom}(s^i(c_0))$ is proven for every $i \geq 0$ by the single ground atom $\textit{Dom}(c_0)$ and Rules of Type~\eqref{rule:domain}.
\item $\textit{Tile}_{f(i, j)}(s^i(c_0), s^j(c_0))$ is proven for every $i, j \geq 0$ by Rules of Type~\eqref{rule:tile-choice}.
\item $\textit{Below}_{t_0}(s^j(c_0))$ is proven for every $j \geq 0$ by Rules of Type~\eqref{rule:below} or \eqref{rule:below-prop} and the fact that 
  there are is an infinite set $S \subseteq \Naturals$ with $f(0, j) = t_0$ for every $j \in S$.
\end{itemize}
\end{proof}

\begin{lemma}
\label{lemma:reduction-hardness-asp-to-tiling}
If the program $P_\mathfrak{T}$ admits an answer set for some tiling system $\mathfrak{T}$, then $\mathfrak{T}$ admits a recurring solution.
\end{lemma}
\begin{proof}
Consider some tiling system $\mathfrak{T} = \langle T, \textit{HI}, \textit{VI}, t_0 \rangle$ such that $P_\mathfrak{T}$ admits an answer set $I$.
Then, we argue that:
\begin{itemize}
\item The set of all terms in $I$ is $\{s^i(c_0) \mid i \geq 0\}$.
Note that $\textit{Dom}(c_0), \textit{Dom}(s(X)) \leftarrow \textit{Dom}(X) \in P_\mathfrak{T}$.
\item For every $i, j \geq 0$, there is exactly one tile $t \in T$ such that $\textit{Tile}_t(s^i(c_0), s^j(c_0)) \in I$ because of rules of Type~\eqref{rule:tile-choice}
  such that constraints in $\textit{HI}$ and $\textit{VI}$ are not violated by rules of Type~\eqref{rule:h-incompatibility} and \eqref{rule:v-incompatibility}.
\item Rules of Type~\eqref{rule:below}, \eqref{rule:below-prop}, and \eqref{rule:under-t0} ensure that for every $j \geq 0$, there is a $j' \geq j$ with $\textit{Tile}_{t_0}(c_0, s^{j'}(c_0)) \in I$. 
  In other words, there is an infinite set $S \subseteq \Naturals$ such that $\textit{Tile}_{t_0}(c_0, s^{j}(c_0)) \in I$ for every $j \in S$. 
\end{itemize}
For every $i, j \geq 0$, let $f(i, j) = t$ where $t$ is the only element in $T$ with $\textit{Tile}_t(s^i(c_0), s^j(c_0)) \in I$.
As per Definition~\ref{definition:tiling-system}, the existence of $f$ implies that $P_\mathfrak{T}$ admits a recurring solution.
\end{proof}

%% file: sections/c-proofs-termination.tex
\section{Proofs for Section~\ref{sec:termination}}\label{sec:appendix-termination}

\subsection{Turing Machine Simulation for Lemma~\ref{lem:asp-finitely-many-finite-as-sigma-2-0-hard}}

We give details about how to emulate the computation of a Turing machine on all inputs with a program, as we discuss in the proof of Lemma~\ref{lem:asp-finitely-many-finite-as-sigma-2-0-hard}.

\begin{definition}
  A (Turing) machine (TM) is a tuple $\langle Q, \delta, q_s, q_a, q_r \rangle$ where $Q$ is a set of states, $\delta$ is a total function from $Q \setminus \{q_a, q_r\} \times \{0,1,B\}$ to $Q \times \{0,1,B\} \times \{L, R\}$, $q_s \in Q$ is the starting state, $q_a \in Q$ is the accepting state, and $q_r \in Q$ is the rejecting state.
The machine $M$ halts on a word $w \in \{0, 1\}^*$ if the computation of $M$ on input $w$ reaches the accepting or rejecting state.
\end{definition}
We assume a binary input alphabet w.l.o.g. while the working tape of the TM additionally uses a blank symbol. The tape is only one-way infinite to the right and the Turing machine bumps on the left, i.e. 
when the TM tries to go to the left beyond the first tape cell, it just stays on the first tape cell.

\begin{definition}
For a machine $M = \langle Q, \delta, q_s, q_a, q_r \rangle$, let $P_M$ be the program that contains the atoms $\mathit{H}_{q_s}(c)$ and $\mathit{Input}(c)$, as well as all of the following rules:
\begin{align}
  \mathit{Right}(X, r(X)) &\leftarrow \mathit{Input}(X), \neg \mathit{Last}(X) \label{input-first}\\
\mathit{Last}(X) &\leftarrow \mathit{Input}(X), \neg \mathit{Right}(X, r(X)) \\
\mathit{Input}(r(X)) &\leftarrow \mathit{Input}(X), \mathit{Right}(X, r(X)) \\
\mathit{FiniteInput} &\leftarrow \mathit{Input}(X), \mathit{Last}(X) \\
&\leftarrow \neg \mathit{FiniteInput} \\
  \mathit{S}_B(X) &\leftarrow \mathit{Last}(X) \\
  \mathit{S}_0(X) &\leftarrow \mathit{Input}(X), \lnot \mathit{S}_1(X), \lnot \mathit{S}_B(X) \\
  \mathit{S}_1(X) &\leftarrow \mathit{Input}(X), \lnot \mathit{S}_0(X), \lnot \mathit{S}_B(X) \label{input-last} \\
\mathit{Step}(Y, s(Y)) &\leftarrow \mathit{Right}(X, Y), \mathit{Step}(X, s(X)) \label{emu-first} \\
\mathit{Step}(Y, s(Y)) &\leftarrow \mathit{Right}(Y, X), \mathit{Step}(X, s(X)) \\
\mathit{Right}(s(X), s(Y)) &\leftarrow \mathit{Step}(X, s(X)), \mathit{Right}(X, Y) \\
  \mathit{Right}(s(X), r(s(&X))) \leftarrow \mathit{Step}(X, s(X)), \mathit{Last}(X) \\
  \mathit{Last}(r(s(X))) &\leftarrow \mathit{Step}(X, s(X)), \mathit{Last}(X) \\
\mathit{Halt} &\leftarrow \mathit{H}_{q_a}(X) \\
\mathit{Halt} &\leftarrow \mathit{H}_{q_r}(X)\\
&\leftarrow \neg \mathit{Halt}
\end{align}
Moreover, for every $q \in Q \setminus \{q_a, q_r\}$, we add the rule
\begin{align}
\mathit{Step}(X, s(X)) \leftarrow \mathit{H}_q(X) 
\end{align}
  For every $a \in \{0,1,B\}$, we add
\begin{align}
  \mathit{S}_a(s(X)) \leftarrow \{\neg &\mathit{H}_q(X) \mid q \in Q\}, \notag\\ 
    &\mathit{S}_a(X), \mathit{Step}(X, s(X))
\end{align}
For every $(q, a) \mapsto (\_, b, \_) \in \delta$, we add
\begin{align}
\mathit{S}_b(s(X)) \leftarrow \mathit{H}_q(X), \mathit{S}_a(X)
\end{align}
For every $(q, a) \mapsto (r, \_, L) \in \delta$, we add
\begin{align}
  \mathit{H}_r(Y) &\leftarrow \mathit{Right}(Y, s(X)), \mathit{H}_q(X), \mathit{S}_a(X)\\
  \mathit{NotFirst}(X) &\leftarrow \mathit{Right}(Y, X), \mathit{H}_q(X), \mathit{S}_a(X)\\
  \mathit{H}_r(s(X)) &\leftarrow \neg \mathit{NotFirst}(X), \mathit{H}_q(X), \mathit{S}_a(X)
\end{align}
For every $(q, a) \mapsto (r, \_, R) \in \delta$, we add
\begin{align}
  \mathit{H}_r(Y) \leftarrow \mathit{Right}(s(X), Y), \mathit{H}_q(X), \mathit{S}_a(X) \label{emu-last}
\end{align}

%  &\text{for each $\delta(q, u) = (q', \_, R)$:}\\
%  &\mathit{H}_{q'}(Z) \leftarrow \mathit{Step}(Y, Z), \mathit{H}_q(X), \mathit{Next}(X, Y).\\
%  &\text{for each $\delta(q, u) = (q', \_, L)$:}\\
%  &\mathit{H}_{q'}(Z) \leftarrow \mathit{Step}(Y, Z), \mathit{H}_q(X), \mathit{Next}(Y, X).\\

\end{definition}

Consider a machine $M$ and the following remarks:
\begin{enumerate}
\item Because of Rules (\ref{input-first}-\ref{input-last}), we have that for every answer set $I$ of $P_M$ there is some (finite) word $w_1, \ldots, w_n$ over the binary alphabet $\{0,1\}$ such that $I$ includes $\{S_{w_1}(c), \allowbreak S_{w_2}(r(c)),  \allowbreak \ldots, \allowbreak S_{w_n}(r^n(c)), \allowbreak S_{B}(r^{n+1}(c))\}$.
Hence, we can associate every answer set $I$ of $P_M$ with a word, which we denote with $\textsf{Word}(I)$.
\item Consider some word $\vec{w}$ over $\{0,1\}$ such that $M$ halts on $\vec{w}$. Then, there is a finite answer set for $P_M$; namely, the only answer set $I$ with $\textsf{Word}(I) = \vec{w}$.
\item Consider some answer set $I$ for $P_M$.
Then, we can show that $M$ halts on $\textsf{Word}(I)$.
\item The previous two items hold because Rules~(\ref{emu-first}-\ref{emu-last}) ensure that every answer set of $P_M$ encodes the computation of $M$ on some input word over the binary alphabet.
Namely, an answer set $I$ faithfully encodes the computation of $M$ on $\textsf{Word}(I)$.
\end{enumerate}

\subsection{Proof of Theorem~\ref{trm:asp-only-finite-as-pi-1-1}}

Theorem~\ref{trm:asp-only-finite-as-pi-1-1} follows from the upcoming Lemmas~\ref{lem:asp-only-finite-as-pi-1-1-mem} and \ref{lem:asp-only-finite-as-pi-1-1-hard}.

\begin{definition}
  For a program $P$, let $\mathcal{M}'_P$
  be a modified version of $\mathcal{M}_P$ (from the Section~\ref{sec:consistency})
  that does not loop on its start state if $L_i^+$ is an answer set in step \ref{loop} but just halts. 
\end{definition}

\begin{lemma}\label{lem:asp-only-finite-as-pi-1-1-mem}
  Deciding if an ASP program is frugal is in $\Pi_1^1$.
\end{lemma}
\begin{proof}
  A program $P$ admits only finite answer sets iff $\mathcal{M}'_P$ does not admit an infinite run that visits its start state infinitely many times. 
  
  If $P$ has an infinite answer set, then by the proof of Lemma~\ref{lemma:reduction-membership-1}, we already know 
  that $\mathcal{M}_P$ admits an infinite run that visits its start state infinitely many times without encountering a finite answer set in this run.
  The same run is also possible in $\mathcal{M}'_P$ by construction.

  If $\mathcal{M}'_P$ admits an infinite run that visits its start state infinitely many times, 
  and since $\mathcal{M}'_P$ does not loop on its start state if it encounters a finite answer set, 
  $P$ must have an infinite answer set as argued in the proof of Lemma~\ref{lemma:reduction-membership-2}.
\end{proof}

\begin{lemma}\label{lem:asp-only-finite-as-pi-1-1-hard}
  Deciding if an ASP program is frugal is $\Pi_1^1$-hard.
\end{lemma}
\begin{proof}
  We can make direct use of the tiling system $\mathfrak{T}$ from the proof of Lemma~\ref{lemma:hardness-reduction} 
  since if $\mathfrak{T}$ has a solution, then $P_\mathfrak{T}$ has an infinite model
  but otherwise, the program does not have any model (i.e. all (0) models are finite).
\end{proof}

%% file: sections/d-proofs-reasoning-procedures.tex
\section{Proofs for Section~\ref{sec:reasoning-procedures}}\label{sec:appendix-reasoning-procedures}

\subsection{Proof of Lemma~\ref{lem:PClosureContainedInEveryAS}}

We show Lemma~\ref{lem:PClosureContainedInEveryAS} as follows.
\begin{proof}
  Assume that the precondition holds, i.e. we have a program $P$ and 
  an answer set $I$ of $P$ 
  and two interpretations $L^+, L^-$ such that 
  $L^+ \subseteq I$ and $L^- \cap I = \emptyset$. 
  We prove 
  that $P_\infty^+(L^+, L^-) \subseteq I$ and $P_\infty^-(L^+, L^-) \cap I = \emptyset$
  by showing the respective property 
  for each $P_i^+(L^+, L^-)$ and $P_i^-(L^+, L^-)$
  via induction over $i \geq 0$.

  For the base of the induction ($i = 0$), the claim follows directly from the precondition; 
  that is, $P_0^+(L^+, L^-) = L^+ \subseteq I$ and $P_0^-(L^+, L^-) \cap I = \emptyset$ since $P_0^-(L^+, L^-) = L^-$.

  We show the induction step from $i$ to $i+1$. 
  Let $J^+ = P_i^+(L^+, L^-)$ and $J^- = P_i^-(L^+, L^-)$.
  By induction hypothesis, we have $J^+ \subseteq I$ and $J^- \cap I = \emptyset$.
  
  For the claim for $P_{i+1}^+(L^+, L^-)$,
  suppose for a contradiction that there is a rule $r \in P$ and an atom $a \in r^+(J^+, J^-)$ with $a \notin I$. 
  Then, by the definition of $r^+$, 
  there are two cases:
  \begin{enumerate}
    \item $a$ is the single atom in $H_r\sigma$ for some substitution $\sigma$ with $B_r^+\sigma \subseteq J^+$ and $B_r^-\sigma \subseteq J^-$, or 
    \item $a$ is the single atom in $B_r^-\sigma$ for some substitution $\sigma$ with $B_r^+\sigma \subseteq J^+$ and $H_r\sigma \subseteq J^-$ with $\lvert B_r^- \rvert = 1$.
  \end{enumerate}
  By the supposition that $a \notin I$ and the induction hypothesis, we obtain that $I$ does not satisfy the ground rule $r\sigma$. 
  This contradicts $I$ being an answer set of $P$.

  For the claim for $P_{i+1}^-(L^+, L^-)$, 
  suppose for a contradiction that there is a rule $r \in P$ and an atom $a \in r^-(J^+, J^-)$ with $a \in I$. 
  By the definition of $r^-$, 
  $a$ is the single atom in $B_r^+\sigma$ for some substitution $\sigma$ with $B_r^-\sigma \subseteq J^-$ and $H_r\sigma \subseteq J^-$ with $\lvert B_r^+ \rvert = 1$.
  By the supposition that $a \in I$ and the induction hypothesis, we obtain that $I$ does not satisfy the ground rule $r\sigma$. 
  This contradicts $I$ being an answer set of $P$.
  This concludes the induction step and the proof.
\end{proof}

\subsection{Proof of Theorem~\ref{trm:forbiddenCheckCorrect}}

We show Theorem~\ref{trm:forbiddenCheckCorrect} as follows.
\begin{proof}
  We prove that $a$ is forbidden, i.e. that $a$ cannot occur in any answer set of $P$, 
  if $\textsf{IsForbidden}(P, \{a\}, \emptyset)$ returns true.
  More precisely, we show the contrapositive, i.e. if $a$ occurs in some answer set of $I$, then 
  $\textsf{IsForbidden}(P, \{a\}, \emptyset)$ returns false.

  Let $I$ be an answer set for $P$.
  We establish the following more general claim ($\dagger$) over the execution of one (recursive) call of \textsf{IsForbidden}. 
  Given two interpretations $L^+, L^-$ 
  such that there is a mapping $h$ from fresh constants to terms with $h(L^+) \subseteq I$ and $h(L^-) \cap I = \emptyset$, 
  we show that 
  either, $\textsf{IsForbidden}(P, L^+, L^-)$ returns false immediately or 
  there is a similar mapping $h'$ with $h'(J^+) \subseteq I$ and $h'(J^-) \cap I = \emptyset$ 
  where $J^+$ and $J^-$ are the inputs to the recursive call.
  The recursion is finite by intersecting with $\textit{TA}^P(K^+ \cup K^-)$ in line~\ref{algoLine:extensionUpdate}, 
  and making sure $g(a)$ only features terms in $L^+ \cup L^-$ in line~\ref{algoLine:ifStillFresh}.
  This eventually also means that $\textsf{IsForbidden}(P, L^+, L^-)$ returns false.

  We show in the following that ($\ddagger$) Line~\ref{algoLine:PClosure} retains the precondition after updating $L^+$ and $L^-$, that is, we still have 
  $h(L^+) \subseteq I$ and $h(L^-) \cap I = \emptyset$.
  For $s \in \{+,-\}$ and $n \in \Naturals \cup \{\infty\}$, 
  let $\mathit{Pout}_n^s := h(P_n^s(L^+, L^-))$
  and $\mathit{Pin}_n^s := P_n^s(h(L^+), h(L^-))$. 
  Now the precondition still holds by Lemma~\ref{lem:PClosureContainedInEveryAS}
  and the observation that 
  $\mathit{Pout}_\infty^s \subseteq \mathit{Pin}_\infty^s$ for each $s \in \{+,-\}$.
  To see why the observation holds, consider the following argument. 
  We show $\mathit{Pout}_i^+ \subseteq \mathit{Pin}_i^+$ via induction over $i$. The proof for $s = -$ is analogous.
  For the base case with $i = 0$, $\mathit{Pout}_i^+ = h(L^+) \subseteq h(L^+) = \mathit{Pin}_i^+$ holds trivially. 
  For the induction step from $i$ to $i+1$, we have 
  $\mathit{Pout}_i^+ \subseteq \mathit{Pin}_i^+$ by induction hypothesis. 
  Note also that $h(\textit{TA}^P(L^+ \cup L^-)) \subseteq \textit{TA}^P(h(L^+) \cup h(L^-))$.
  Now let $\mathit{rOut}^+ := h(r^+(P_i^+(L^+, L^-), P_i^-(L^+, L^-)))$
  and $\mathit{rIn}^+ := r^+(P_i^+(h(L^+), h(L^-)), P_i^-(h(L^+), h(L^-)))$.
  It only remains to show that $\mathit{rOut}^+ \subseteq \mathit{rIn}^+$ for each $r \in P$.
  For every atom in $a' \in \mathit{rOut}^+$, there are two cases to consider: 
  \begin{itemize}
    \item There is a substitution $\sigma$ such that $a'$ is the single atom $h(H_r\sigma)$, $B_r^+\sigma \subseteq P_i^+(L^+, L^-)$, and $B_r^-\sigma \subseteq P_i^-(L^+, L^-)$. 
      Then, $B_r^+(h \circ \sigma) \subseteq \mathit{Pout}_i^+ \subseteq \mathit{Pin}_i^+$ and $B_r^-(h \circ \sigma) \subseteq \mathit{Pout}_i^- \subseteq \mathit{Pin}_i^-$. 
      Therefore, $H_r(h \circ \sigma) = h(H_r\sigma) \subseteq \mathit{rIn}^+$.
    \item There is a substitution $\sigma$ such that $a'$ is the single atom in $h(B_r^-\sigma)$, $B_r^+\sigma \subseteq P_i^+(L^+, L^-)$, and $H_r\sigma \subseteq P_i^-(L^+, L^-)$.
      Then, $B_r^+(h \circ \sigma) \subseteq \mathit{Pout}_i^+ \subseteq \mathit{Pin}_i^+$ and $H_r(h \circ \sigma) \subseteq \mathit{Pout}_i^- \subseteq \mathit{Pin}_i^-$. 
      Therefore, $B_r^-(h \circ \sigma) = B_r^-(h \circ \sigma) \subseteq \mathit{rIn}^+$.
  \end{itemize}
  This concludes the induction step and the proof of ($\ddagger$).

  By the precondition, i.e. $h(L^+) \subseteq I$ and $h(L^-) \cap I = \emptyset$, we have 
  $h(L^+) \cap h(L^-) = \emptyset$ and therefore also $L^+ \cap L^- = \emptyset$. 
  Thus, we do not reach line~\ref{algoLine:true}.

  If all atoms in $L^+$ have support, we do not enter the loop in line~\ref{algoLine:loopAtoms}. 
  Then, the algorithm returns false as claimed. 
  Otherwise, 
  let $a' \in L^+$ be any atom picked in line~\ref{algoLine:loopAtoms}.
  Since $h(a') \in I$, $h(a')$ is proven in $I$ and $P$. 
  Therefore, there is a suitable $r \in P$ and a substitution $\sigma$ 
  with $H_r\sigma = \{ h(a') \}$ in line~\ref{algoLine:loopRules}. 
  We pick $g$ to be $h$.
  Since we already established that $h(a') \in I$, we know that $h(a')$ does not contain any fresh constants. 
  However, $h(a')$ might contain terms that are not in $L^+ \cup L^-$. 
  In this case, we enter the condition on line~\ref{algoLine:ifStillFresh}
  and return false as claimed.
  Otherwise, we proceed as follows.
  
  For what follows, note that $h$ is idempotent; fresh constants are not mapped to other fresh constants and all other terms are not mapped (i.e. effectively mapped to themselves). 
  Therefore, in line~\ref{algoLine:gUpdate}, $h(K^+) = h(h(L^+)) = h(L^+) \subseteq I$ holds and $h(K^-) = h(h(L^-)) = h(L^-)$ so $h(K^-) \cap I = \emptyset$ holds as well.
  Again, since $h(a')$ is proven in $I$ and $P$, there is a substitution $\sigma''$ 
  with $H_r\sigma = H_r\sigma''$, $B_r^+\sigma'' \subseteq I$, and $B_r^-\sigma'' \cap I = \emptyset$.
  In line~\ref{algoLine:freshConst}, we pick $\sigma'$ to be the $r$-extension of $\sigma$ with $\sigma'(X) = \sigma''(X)$ for every body variable $X$ 
  that occurs in a position that can only feature constants (and $\sigma'(Y)$ being a fresh constant all other variables $Y$).

  To wrap up ($\dagger$) it remains to show that we find a mapping $h'$ from fresh constants to terms with $h'(J^+) \subseteq I$ and $h'(J^-) \cap I = \emptyset$.
  We define $h'$ as an extension of $h$ additionally mapping the newly introduced constants in $\sigma'$ 
  such that $\sigma'' = h' \circ \sigma'$. 
  We obtain $h'(K^+) = h(K^+)$ and $h'(K^-) = h(K^-)$.
  Hence, according to line~\ref{algoLine:extensionUpdate}, 
  the proof of ($\dagger$) concludes once we prove $h'(B_r^+\sigma') \subseteq I$ and $h'(B_r^-\sigma') \cap I = \emptyset$.
  This is straightforward since $h'(B_r^+\sigma') = B_r^+\sigma''$ and $h'(B_r^-\sigma') = B_r^-\sigma''$
  and we already know that $B_r^+\sigma'' \subseteq I$ and $B_r^-\sigma'' \cap I = \emptyset$ hold.

  The claim of the theorem now simply follows by applying ($\dagger$) to $\{a\}$ and $\emptyset$. 
  We pick $h$ to be the identity. Hence with the assumption that $a \in I$, the precondition holds 
  and we therefore infer that 
  $\textsf{IsForbidden}(P, \{a\}, \emptyset)$ returns false.
\end{proof}

\subsection{Run of Algorithm~\ref{alg:anyForbidden} for Example~\ref{exp:essential-wgc-different}}

We show how Algorithm~\ref{alg:anyForbidden} verifies that $\textit{fct}(a, s(s(0)))$ is forbidden in Example~\ref{exp:essential-wgc-different}.

\begin{itemize}
  \item Initialize $L^+$ with $\textit{fct}(a, s(s(0)))$ and $L^-$ with $\emptyset$.
  \item In line~\ref{algoLine:PClosure}, we obtain   
    $P_\infty^-(L^+, L^-) = \{ \textit{redundant} \}$ as well as
    $P_\infty^+(L^+, L^-) = \{ \textit{fct}(a, 0), \allowbreak \textit{fct}(b, s(0)), \allowbreak \textit{eq}(a, a), \allowbreak \textit{eq}(b, b), \allowbreak \textit{lt}(0, s(0)), \allowbreak \textit{lt}(s(0), s(s(0))), \allowbreak \textit{lt}(0, s(s(0))), \allowbreak \textit{diff}(0, s(0)), \allowbreak \textit{diff}(s(0), s(s(0))), \allowbreak \textit{diff}(0, s(s(0))) \}$.
  \item In the loop in line~\ref{algoLine:loopAtoms}, pick $\textit{diff}(0, s(s(0)))$.
  \item In the loop in line~\ref{algoLine:loopRules}, 
    there is only one possible choice with $r$ being the next to last rule, $g$ being the identity, 
    and $\sigma$ mapping $N$ to $0$ and $M$ to $s(s(0))$. 
  \item For the $r$-extension of $\sigma$ in line~\ref{algoLine:freshConst}, there are four possible choices that we consider individually. 
    Both $X$ and $Y$ can each be mapped to $a$ or $b$. No fresh constants are involved since the first position of $\textit{fct}$ may only feature constants, namely $a$ or $b$.
  \begin{enumerate}
    \item $X \mapsto a$ and $Y \mapsto a$. 
        \begin{itemize}
          \item In line~\ref{algoLine:extensionUpdate}, we add $\textit{eq}(a, a)$ to $J^-$. 
          \item We reach line~\ref{algoLine:true} since $\textit{eq}(a, a) \in L^+ \cap L^-$.
        \end{itemize}
    \item $X \mapsto b$ and $Y \mapsto b$. 
        \begin{itemize}
          \item In line~\ref{algoLine:extensionUpdate}, add $\textit{eq}(b, b)$ to $J^-$. ($J^+$ also changes.) 
          \item We reach line~\ref{algoLine:true} since $\textit{eq}(b, b) \in L^+ \cap L^-$.
        \end{itemize}
    \item $X \mapsto b$ and $Y \mapsto a$. 
        \begin{itemize}
          \item In line~\ref{algoLine:extensionUpdate}, add $\textit{fct}(b, 0)$ to $J^+$ and $\textit{eq}(b, a)$ to $J^-$. 
          \item In line~\ref{algoLine:loopAtoms}, we pick $\textit{fct}(b, 0)$.
          \item We cannot enter the loop in line~\ref{algoLine:loopRules}; therefore we return true in the end.
        \end{itemize}
    \item $X \mapsto a$ and $Y \mapsto b$. 
        \begin{itemize}
          \item In line~\ref{algoLine:extensionUpdate}, we add $\textit{fct}(b, s(s(0)))$ to $J^+$ and $\textit{eq}(a, b)$ to $J^-$. 
          \item In line~\ref{algoLine:loopAtoms}, we pick $\textit{fct}(b, s(s(0)))$.
          \item In the loop in line~\ref{algoLine:loopRules}, 
            there is only one possible choice with $r$ being $\textit{fct}(b, s(N)) \leftarrow \textit{fct}(a, N).$, $g$ being the identity, 
            and $\sigma$ mapping $N$ to $s(0)$. 
          \item In line~\ref{algoLine:freshConst}, $\sigma' = \sigma$.
          \item In line~\ref{algoLine:extensionUpdate}, we add $\textit{fct}(a, s(0))$ to $J^+$. 
          \item In line~\ref{algoLine:loopAtoms}, we pick $\textit{fct}(a, s(0))$.
          \item In the loop in line~\ref{algoLine:loopRules}, 
            there is only one possible choice with $r$ being $\textit{fct}(a, s(N)) \leftarrow \textit{fct}(b, N).$, $g$ being the identity, 
            and $\sigma$ mapping $N$ to $0$. 
          \item In line~\ref{algoLine:freshConst}, $\sigma' = \sigma$.
          \item In line~\ref{algoLine:extensionUpdate}, we add $\textit{fct}(b, 0)$ to $J^+$. 
          \item Now we eventually return true as in case 3.
        \end{itemize}
  \end{enumerate}
\end{itemize}

\subsection{Proof of Theorem~\ref{trm:groundWithoutForbiddenValid}}

We show Theorem~\ref{trm:groundWithoutForbiddenValid} as follows.
\begin{proof}
  Let $P_g$ be the result of $\textsf{GroundNotForbidden}(P)$. 
  For every answer set $I$ of $\textsf{Ground}(P)$, 
  we have that $I \subseteq \bigcup_{i \geq 0} A_i$ (with the $A_i$ from the construction of $P_g$). 
  This holds, since every atom $a \in I$ is proven and not forbidden (as it occurs in an answer set),
  so there is an $i$ with $a \in A_i$.

  So if $I$ is an answer set of $\textsf{Ground}(P)$, then 
  all atoms in $I$ are still proven by $P_g$ and rules in $\textsf{Ground}(P) \cap P_g$ are still satisfied.
  It only remains to show that all rules in $P_g \setminus \textsf{Ground}(P)$ are satisfied.
  Such rules $r$ must be of the form $\leftarrow B_{r'}$ introduced for rules 
  $r' \in \textsf{Ground}(P)$ where $H_{r'}$ is forbidden. 
  If $I$ would not satisfy $r$, it would also not satisfy $r'$ unless $H_{r'} \in I$, 
  which contradicts $I$ being an answer set of $\textsf{Ground}(P)$.
  This completes the ``only if'' direction.

  If $I'$ is an answer set of $P_g$, 
  every atom $a \in I'$ is still proven in $\textsf{Ground}(P)$. 
  It only remains to show that $I'$ is indeed a model of $\textsf{Ground}(P)$. 
  Suppose for a contradiction that $r \in \textsf{Ground}(P)$ is not satisfied by $I'$. 
  By construction of $P_g$, this can only be the case if all atoms in $H_r$ are forbidden. 
  But then, the rule $\leftarrow~B_r.$ in $P_g$ is also not satisfied by $I'$, 
  which contradicts $I'$ being an answer set of $P_g$.
  This completes the ``if'' direction.
\end{proof}